\documentclass[twoside,11pt]{article}

\usepackage{blindtext}
\usepackage{subcaption}
\usepackage{booktabs}
\usepackage{float}
\usepackage{mathtools}
\usepackage{tikz}
\usepackage[normalem]{ulem}
\usetikzlibrary{positioning}
\usepackage{tcolorbox}

\usepackage[nohyperref]{jmlr2e}

\usepackage{amsmath,amsfonts,bm}
\usepackage{bbm}
\usepackage[mathscr]{eucal}

\renewcommand{\vec}[1]{{\bm{#1}}}
\newcommand{\mat}[1]{\mathbf{#1}}
\newcommand{\Sph}{\mathbb{S}}

\def\1{\bm{1}}

\def\<{\langle}
\def\>{\rangle}

\def\va{{\bm{a}}}
\def\vb{{\bm{b}}}

\def\vv{{\bm{v}}}
\def\vw{{\bm{w}}}
\def\vx{{\bm{x}}}
\def\vy{{\bm{y}}}
\def\vz{{\bm{z}}}

\DeclareMathAlphabet{\mathsfit}{\encodingdefault}{\sfdefault}{m}{sl}
\SetMathAlphabet{\mathsfit}{bold}{\encodingdefault}{\sfdefault}{bx}{n}

\def\cL{{\mathcal{L}}}

\def\bR{{\mathbb{R}}}

\newcommand{\V}{\mathcal{V}}

\newcommand{\R}{\mathbb{R}}

\DeclareMathOperator*{\argmax}{arg\,max}
\DeclareMathOperator*{\argmin}{arg\,min}

\DeclarePairedDelimiter\abs{\lvert}{\rvert}
\DeclarePairedDelimiter\curly{\{}{\}}
\DeclarePairedDelimiter\paren{(}{)}
\DeclarePairedDelimiter\sq{[}{]}
\DeclarePairedDelimiter\norm{\lVert}{\rVert}
\DeclarePairedDelimiter\ang{\langle}{\rangle}

\newcommand{\dd}{\,\mathrm d}
\newcommand{\T}{T}

\def\st{\:\colon\:}

\newcommand{\M}{\mathcal{M}}

\newcommand{\dummy}{\mkern1mu\cdot\mkern1mu}

\DeclareFontEncoding{FMS}{}{}
\DeclareFontSubstitution{FMS}{futm}{m}{n}
\DeclareFontEncoding{FMX}{}{}
\DeclareFontSubstitution{FMX}{futm}{m}{n}
\DeclareSymbolFont{fouriersymbols}{FMS}{futm}{m}{n}
\DeclareSymbolFont{fourierlargesymbols}{FMX}{futm}{m}{n}
\DeclareMathDelimiter{\VERT}{\mathord}{fouriersymbols}{152}{fourierlargesymbols}{147}

\newcommand\numberthis{\addtocounter{equation}{1}\tag{\theequation}}

\newcommand{\aug}[1]{\overline{#1}}

\DeclareMathOperator{\HOp}{H}
\DeclareMathOperator{\TV}{TV}
\newcommand{\iso}{\mathrm{iso}}

\let\hat\relax
\newcommand{\hat}{\widehat}

\let\tilde\relax
\newcommand{\tilde}{\widetilde}

\DeclareMathOperator*{\esssup}{ess\,sup}
\DeclareMathOperator{\spn}{span}

\usepackage{xcolor}

\newcommand{\rob}[1]{}
\newcommand{\joe}[1]{}
\newcommand{\rahul}[1]{}

\renewcommand{\rob}[1]{{\color{red}{RN: #1}}}
\renewcommand{\joe}[1]{{\color{blue}{JS: #1}}}
\renewcommand{\rahul}[1]{{\color{orange}{rahul: #1}}}

\usepackage{lastpage}
\jmlrheading{25}{2024}{1-\pageref{LastPage}}{5/23; Revised
3/24}{6/24}{23-0677}{Joseph Shenouda, Rahul Parhi, Kangwook Lee, and Robert D. Nowak}

\ShortHeadings{Variation Spaces for Multi-Output Neural Networks}{Shenouda, Parhi, Lee, and Nowak}
\firstpageno{1}

\usepackage{hyperref}
\usepackage[nameinlink, capitalize]{cleveref}
\hypersetup{bookmarksopen=true}

\usepackage{crossreftools}
\crefformat{equation}{\textup{#2(#1)#3}}
\crefrangeformat{equation}{\textup{#3(#1)#4--#5(#2)#6}}
\crefmultiformat{equation}{\textup{#2(#1)#3}}{ and \textup{#2(#1)#3}}
{, \textup{#2(#1)#3}}{, and \textup{#2(#1)#3}}
\crefrangemultiformat{equation}{\textup{#3(#1)#4--#5(#2)#6}}%
{ and \textup{#3(#1)#4--#5(#2)#6}}{, \textup{#3(#1)#4--#5(#2)#6}}{, and \textup{#3(#1)#4--#5(#2)#6}}

\Crefformat{equation}{#2Equation~\textup{(#1)}#3}
\Crefrangeformat{equation}{Equations~\textup{#3(#1)#4--#5(#2)#6}}
\Crefmultiformat{equation}{Equations~\textup{#2(#1)#3}}{ and \textup{#2(#1)#3}}
{, \textup{#2(#1)#3}}{, and \textup{#2(#1)#3}}
\Crefrangemultiformat{equation}{Equations~\textup{#3(#1)#4--#5(#2)#6}}%
{ and \textup{#3(#1)#4--#5(#2)#6}}{, \textup{#3(#1)#4--#5(#2)#6}}{, and \textup{#3(#1)#4--#5(#2)#6}}

\crefname{remark}{Remark}{Remarks}
\crefname{lemma}{Lemma}{Lemmas}
\crefname{definition}{Definition}{Definitions}
\crefname{corollary}{Corollary}{Corollaries}
\crefname{proposition}{Proposition}{Propositions}
\crefname{figure}{Figure}{Figures}

\hypersetup{hidelinks}

\begin{document}
\title{Variation Spaces for Multi-Output Neural Networks: \\
 Insights on Multi-Task Learning and Network Compression}

\author{\name Joseph Shenouda \email jshenouda@wisc.edu \\
       \addr Department of Electrical and Computer Engineering\\
       University of Wisconsin--Madison\\
       Madison, WI 53706, USA
       \AND
       \name Rahul Parhi\thanks{Part of this work was done while RP was with the University of Wisconsin--Madison and then with the \'Ecole polytechnique f\'ed\'erale de Lausanne. He is now with the University of California, San Diego.} \email rparhi@ucsd.edu \\
       \addr Department of Electrical and Computer Engineering \\
        University of California, San Diego \\
       La Jolla, CA 92093, USA
       \AND
       \name Kangwook Lee \email kangwook.lee@wisc.edu \\
       \name Robert D. Nowak \email rdnowak@wisc.edu\\
       \addr Department of Electrical and Computer Engineering\\
       University of Wisconsin--Madison\\
       Madison, WI 53706, USA}

\editor{Aarti Singh}

\maketitle

\begin{abstract}
This paper introduces a novel theoretical framework for the analysis of vector-valued neural networks through the development of vector-valued variation spaces, a new class of reproducing kernel Banach spaces. These spaces emerge from studying the regularization effect of weight decay in training networks with activation functions like the rectified linear unit (ReLU). This framework offers a deeper understanding of multi-output networks and their function-space characteristics. A key contribution of this work is the development of a representer theorem for the vector-valued variation spaces. This representer theorem establishes that shallow vector-valued neural networks are the solutions to data-fitting problems over these infinite-dimensional spaces, where the network widths are bounded by the square of the number of training data. This observation reveals that the norm associated with these vector-valued variation spaces encourages the learning of features that are useful for multiple tasks, shedding new light on multi-task learning with neural networks. Finally, this paper develops a connection between weight-decay regularization and the multi-task lasso problem. This connection leads to novel bounds for layer widths in deep networks that depend on the intrinsic dimensions of the training data representations. This insight not only deepens the understanding of the deep network architectural requirements, but also yields a simple convex optimization method for deep neural network compression. The performance of this compression procedure is evaluated on various architectures.
\end{abstract}



\section{Introduction}
The investigation of shallow scalar-valued (single-output) neural networks through the lens of function spaces has been extensively developed. This line of work was pioneered by \citet{BarronUniversal,kurkova2001bounds,kurkova2002comparison, mhaskar2004tractability, barron2008approximation, bach2017breaking}, where those authors study the so-called \emph{variation spaces} of shallow networks. On the other hand, the investigation of vector-valued (multi-output) networks is far less developed. To fill that gap, this paper develops a new framework to investigate the characteristics of the functions learned by vector-valued networks. Since each layer of a deep neural network (DNN) is itself a shallow vector-valued neural network, the proposed investigation is crucial to provide insights into learning with \emph{deep} neural networks.  

The developed framework is based on a novel class of Banach spaces which are termed \emph{vector-valued variation spaces}. These spaces arise naturally from a precise characterization of the function-space norm that is regularized by weight decay when training neural networks with homogeneous activation functions, such as the rectified linear unit (ReLU). The vector-valued variation spaces reduce to the classical, scalar-valued variation spaces in the single-output setting. In the general vector-valued setting, these spaces exhibit several intriguing properties that shed light on multi-task learning with neural networks as well as the inductive bias of weight-decay regularization. These properties motivate a simple and computationally efficient method to compress pre-trained deep neural networks (DNNs).

\subsection{Organization and Main Contributions}
This paper is organized as follows. In \cref{sec:nbt} we review the effect of weight-decay regularization in the training of DNNs with homogeneous activation functions. There, it is revealed that weight-decay regularization is equivalent to a constrained form of multi-task lasso regularization. This observation then naturally leads to a variation norm on the space of vector-valued networks and their (appropriately taken) wide limits, which comprises the vector-valued variation space. These vector-valued variation spaces characterize the functions generated by shallow vector-valued networks (or, equivalently, layers of deep networks). 

In \cref{VVspace,sec:vv_rep_thm} we develop the vector-valued variation spaces and investigate some of their properties. We show that the vector-valued variation spaces are ``immune'' to the curse of dimensionality in the sense that any function in a vector-valued variation space can be $\epsilon$-approximated (in $L^2$) by a vector-valued neural network whose width scales as $\epsilon^{-1/2}$, which is independent of both the input and output dimensions. We show that these spaces are reproducing kernel Banach spaces and prove a representer theorem. The representer theorem shows that shallow vector-valued neural networks are solutions to data-fitting problems over the infinite-dimensional vector-valued variation spaces. Furthermore, it guarantees that the widths of these neural network solutions are bounded by the square of the number of training data, irrespective of the input and output dimensions.

In \cref{sec:neuron-sharing} we show that the norm associated with the vector-valued variation spaces (and hence weight decay in the case of homogeneous activation functions) promotes a remarkable ``neuron sharing" property of the solutions. This refers to the fact that the norm encourages solutions in which each neuron contributes to \emph{every} output function, as opposed to networks in which disjoint sets of neurons are used to represent different output functions.
This indicates how different outputs may influence each other in the training process. It also provides a new viewpoint of multi-task learning with neural networks: Solutions are encouraged to learn features that are useful for multiple tasks.

Finally, in \cref{sec:bounds} we show that weight-decay regularization in DNNs with homogeneous activation functions is tightly linked to a convex multi-task lasso problem. With this link, we present new bounds on the sufficient widths of layers in DNNs that depend on the intrinsic dimensions of learned data representations at each layer. This result is based on a novel characterization of the sparsity of multi-task lasso solutions, which may be of independent interest. Notably, if the dimensions (ranks of data representation matrices) are low, then there exist layers whose widths are much narrower than the number of data while still being optimal solutions. This leads to a principled approach to DNN compression. This approach is computationally efficient and can dramatically reduce layer widths without the sacrifice of the learned representations, data-fitting error, or optimality (in terms of the weight decay objective and the variation norm). We evaluate the performance of our proposed compression approach on various architectures in \cref{sec:exp}.
\section{Weight Decay and the Neural Balance Theorem}\label{sec:nbt}
Let $f_{\vec{\theta}}$ be a DNN with weights $\vec{\theta}$ and let $\{(\vec{x}_i,\vec{y}_i)\}_{i=1}^N$ be a dataset where each $\vec{x}_i \in \R^d$ and $\vec{y}_i\in \R^D$.  To fit the data, a common approach is to train the network using gradient descent with \emph{weight decay}. This corresponds to finding a solution to the optimization problem
\begin{equation}\label{opt:general_weight_decay}
    \min_{\vec{\theta}} \sum_{i=1}^{N} \mathcal{L}(\vy_i, f_{\vec{\theta}}(\vx_i)) + \frac{\lambda}{2} \|\vec{\theta}\|^2_{2},
\end{equation}
where $\mathcal{L}(\cdot,\cdot)$ is a loss function that is lower semicontinuous in its second argument and $\lambda > 0$ is the regularization parameter  \citep{hanson1988comparing}. 
Contemporary neural network architectures include a variety of different building blocks, but a commonality among them is neurons with the ReLU activation function. A ReLU neuron is a map of the form $\vec{z}\mapsto \vec{v} (\vec{w}^T\vec{z})_+$ with $\vec{z},\vec{w} \in \R^{d_{\mbox{\tiny in}}}$, $\vec{v}\in \R^{d_{\mbox{\tiny out}}}$, and where $(\cdot)_+ \coloneqq \max\{0,\cdot\}$ is the ReLU.

More generally, an activation function $\sigma: \R \to \R$ is said to be \emph{positively homogeneous of degree one} (or simply homogeneous) if, for any $\gamma > 0$, $\sigma(\gamma t) = \gamma \sigma(t)$ for all $t \in \R$. The ReLU, leaky ReLU, absolute value, and linear activation functions satisfy this property.  A key observation in both theory~\citep{grandvalet1998least,grandvalet1998outcomes,NeyshaburInductiveBias,parhi2023deep} and practice~\citep[Figure~5]{kunin2021neural} is that in any solution to (\ref{opt:general_weight_decay}),
the $2$-norms of the input and output weights of each neuron with a homogeneous activation function must be balanced. This phenomenon is summarized in the \emph{neural balance theorem} (NBT)~\citep{yang2022better,parhi2023deep}.
\begin{theorem}[Neural Balance Theorem]\label{thm:neural_balance}
  Let $f_\vec{\theta}$ be a DNN of any architecture such that $\vec{\theta}$ minimizes \cref{opt:general_weight_decay}. Then, the weights satisfy the following balance constraint: If $\vec{w}$ and $\vec{v}$ denote the input and output weights of any neuron with a homogeneous activation function, then $\norm{\vec{w}}_2 = \norm{\vec{v}}_2$.
\end{theorem}

While the NBT is a simple observation, it allows for an alternative perspective
on weight-decay regularization. Indeed, first consider the weight-decay regularized problem for a shallow vector-valued neural network
whose activation function $\sigma: \R \to \R$ is
homogeneous
\begin{equation}\label{opt:weight-decay-shallow}
 \min_{\{\vec{w}_k,\vec{v}_k\}_{k=1}^K} \sum_{i=1}^{N}\cL\left(\vy_i, \sum_{k=1}^K \vec{v}_k \sigma(\vec{w}_k^\T \aug{\vec{x}}_i)  \right) +
  \frac{\lambda}{2} \sum_{k=1}^{K} \|\vv_k\|_{2}^{2} + \|\vw_k\|_{2}^{2},
\end{equation}
where $\vec{v}_k \in \R^D$, $\vec{w}_k \in \R^{d+1}$ and $\aug{\vec{x}} \coloneqq \sq{\vec{x} \ 1}^\T \in \R^{d+1}$ augments $\vec{x}$ to account for a bias term.
Observe that, by the
NBT, any solution to this weight-decay regularized problem
is always a solution to the so-called \emph{path-norm} regularized problem
\begin{equation}\label{opt:path-norm}
     \min_{\{\vec{w}_k,\vec{v}_k\}_{k=1}^K} \sum_{i=1}^{N}\cL\left(\vy_i, \sum_{k=1}^K \vec{v}_k \sigma(\vec{w}_k^\T \aug{\vec{x}}_i)\right) +
     \lambda \sum_{k=1}^{K} \|\vv_k\|_{2}\|\vw_k\|_{2}.
\end{equation}
Furthermore, thanks to the homogeneity of $\sigma$, any solution to
\cref{opt:path-norm} is always a solution to the
constrained problem
\begin{equation}
    \min_{\{\vec{w}_k,\vec{v}_k\}_{k=1}^K} \: \sum_{i=1}^{N} \mathcal{L}\left(\vec{y}_i,
    \sum_{k=1}^K \vec{v}_k \sigma(\vec{w}_k^\T \aug{\vec{x}}_i)\right) + \lambda \sum_{k=1}^{K} \|\vec{v}_k\|_2
    \quad\mathrm{s.t.}\quad \|\vec{w}_k\|_2=1, \: k=1,\dots,K, \label{opt:path-norm-sphere}
\end{equation}
upon ``absorbing'' the magnitude of the input weights into the output weights.\footnote{That is, by reparameterizing the weights of each neuron as $(\vec{v}_k, \vec{w}_k) \gets (\vec{v}_k \norm{\vec{w}_k}_2, \vec{w}_k / \norm{\vec{w}_k}_2)$, $k = 1, \ldots, K$. Observe that this reparameterization does not change the overall function realized by the neural network.} Clearly, any solution to \cref{opt:path-norm-sphere} is always a solution to \cref{opt:path-norm}. Moreover, any solution to \cref{opt:path-norm} is always a solution to \cref{opt:weight-decay-shallow}, after balancing the weights. Thus, the problems \cref{opt:weight-decay-shallow,opt:path-norm,opt:path-norm-sphere} should be viewed as equivalent. In particular, the regularizer in \cref{opt:path-norm-sphere} is the \emph{multi-task lasso} regularizer which is known to promote a kind of sparsity~\citep{obozinski2006multi,obozinski2010joint,argyriou2008convex}. We also remark that the neural balance theorem and the equivalence between \cref{opt:weight-decay-shallow,opt:path-norm,opt:path-norm-sphere} holds in the case of unregularized biases \citep[cf.][pp. 65--66]{parhi2023deep}. Furthermore, all of the results of this paper also hold in that setting as well. For notational convenience, we focus on the regularized bias scenario (see also \cref{rem:bias}).

The equivalence between \cref{opt:weight-decay-shallow,opt:path-norm,opt:path-norm-sphere} can then be extended to any layer in a deep neural network with homogeneous activation functions (see \cref{sec:width-bounds}). This ``secret sparsity of weight decay'' has many remarkable implications on the understanding of DNNs trained with weight decay \citep[see][for more details]{parhi2023deep}. Furthermore, this connection leads to a natural characterization of the function spaces of vector-valued neural networks, which is developed in the next section.
\section{Vector-Valued Variation Spaces}
\label{VVspace}
The regularization term $\sum_{k=1}^{K} \|\vec{v}_k\|_2$ in the optimization (\ref{opt:path-norm-sphere}) may be viewed as the \emph{representational cost} of a network.  We can adopt this as a measure of the cost of any shallow network or layer in a DNN.
Consider a shallow neural network or network layer of the form
\begin{equation}
    \vec{x} \mapsto \sum_{k=1}^K \vec{v}_k \sigma(\vec{w}_k^\T \aug{\vec{x}}), \quad \vec{x} \in \R^d,
    \label{basic_net}
\end{equation}
with $\vec{w}_k \in \R^{d+1}$ such that $ \|\vec{w}_k\|_2=1$, and $\vec{v}_k \in \R^{D}$, $ k=1,\dots,K$.
If a function $f$ can be represented by a finite-width network, then its representational cost is
\begin{equation}
    \|f\| \coloneqq \inf_{\substack{\{\vec{w}_k,\vec{v}_k\}_{k=1}^K \\ K \in \mathbb{N}}} \sum_{k=1}^{K} \|\vec{v}_k\|_2 \quad\mathrm{s.t.}\quad f = \paren*{\vec{x} \mapsto \sum_{k=1}^K \vec{v}_k \sigma(\vec{w}_k^\T \aug{\vec{x}})}.
    \label{fnorm}
\end{equation}
The $\inf$ arises due to the fact that there are many neural network representations of the same function $f$. The $\inf$ selects the one with the \emph{lowest} representational cost.
The reader can quickly verify that this quantity is a \emph{bona fide} norm since it satisfies the following properties.
\begin{enumerate}
    \item \textbf{Triangle Inequality}: $\|f+g\| \leq \|f\|+\|g\|$.
    \item \textbf{Homogeneity}: $\|\alpha f\| = |\alpha| \|f\|$ for $\alpha \in \R$.
    \item \textbf{Positive Definiteness}: $\|f\|= 0$ if and only if $f = 0$.
\end{enumerate}
The space of all finite-norm neural network functions of the form \cref{basic_net} and their limits\footnote{These are weak${}^*$ limits as opposed to norm limits \citep[cf.,][]{ma2022barron}.}  defines a Banach space that we call a  vector-valued variation space.  These spaces, which capture all functions that can be represented or approximated by neural networks with finite norms, are developed in the sequel. We begin by reviewing past constructions of variation spaces for scalar (single-output) networks.

\subsection{Scalar-Valued Variation Spaces}
In this subsection, we review the definition of the classical, scalar-valued variation spaces.  The results stated here can be found in the papers of~\cite{BengioConvex,bach2017breaking,parhi2021banach,siegel2023characterization}.  The main idea is to consider shallow neural networks with possibly continuously many neurons. These neural networks are parameterized by a finite (Radon) measure. The scalar-valued variation space is the space of functions that map $\R^d \to \R$
\begin{equation}
    \V_{\sigma}(\R^d) \coloneqq \curly*{f(\vec{x}) = \int_{\Sph^d} \sigma(\vec{w}^\T\aug{\vec{x}}) \dd \nu(\vec{w}) \st \vec{x} \in \R^d, \nu \in \M(\Sph^d)},
\end{equation}
where $\Sph^d \coloneqq \{\vec{w} \in \R^{d+1} \st \norm{\vec{w}}_2 = 1\}$ is the unit sphere, $\aug{\vec{x}} \coloneqq \sq{\vec{x} \ 1}^\T \in \R^{d+1}$ augments $\vec{x}$ to account for a bias term, and $\M(\Sph^d)$ is the space of finite (Radon) measures. The measure $\nu$ plays the role of the output weight of each neuron. Here, and in the rest of this paper, the activation function $\sigma: \R \to \R$ will always be assumed to be continuous.

Since each function
$f \in \V_{\sigma}(\R^d)$ is parameterized by a measure $\nu \in \M(\Sph^d)$, we introduce the notation
\begin{equation}
    f_\nu(\vec{x}) \coloneqq \int_{\Sph^d} \sigma(\vec{w}^\T\aug{\vec{x}}) \dd\nu(\vec{w}), \quad \vec{x} \in \R^d.
    \label{eq:scalar-convex-NN}
\end{equation}
It is well-known that the space $\V_{\sigma}(\R^d)$ is a Banach space when equipped with the norm
\begin{equation}
    \norm{f}_{\V_{\sigma}(\R^d)} \coloneqq \inf_{\substack{\nu \in \M(\Sph^d) \\ f = f_\nu}} \norm{\nu}_{\M(\Sph^d)},
    \label{eq:SV-norm}
\end{equation}
where $\norm{\dummy}_{\M(\Sph^d)}$ denotes the total variation norm in the sense of measures.  If $f$ is a finite-width network, this norm is in fact equal to the norm defined \cref{fnorm}, as shown in \cref{eq:scalar-path-norm} below.

As in \cref{fnorm}, the $\inf$ arises since the dictionary of neurons $\curly{\vec{x} \mapsto \sigma(\vec{w}^\T\aug{\vec{x}})}_{\vec{w} \in \Sph^d}$ is highly redundant. Thus, there are many different representations for a given 
$f \in \V_{\sigma}(\R^d)$. By choosing the representation with the smallest total variation norm, \cref{eq:SV-norm} defines a valid Banach norm on 
$\V_{\sigma}(\R^d)$
~\citep[see, e.g.,][Lemma~3]{siegel2023characterization}. Here, we use the following definition of $\norm{\dummy}_{\M(\Sph^d)}$
\begin{equation}
    \norm{\nu}_{\M(\Sph^d)} \coloneqq \sup_{\substack{\Sph^d = \bigcup_{i=1}^n A_i \\ n \in \mathbb{N}}} \sum_{i=1}^n \abs{\nu(A_i)},
    \label{eq:scalar-measure-norm}
\end{equation}
where the $\sup$ is taken over all partitions of $\Sph^d$ (i.e., $A_i \cap A_j = \varnothing$ for $i \neq j$). This definition is equal to the more conventional definitions based on the Jordan decomposition of a measure or as a dual norm~\citep{diestel1977vector,bredies2020higher}. We use the definition in \cref{eq:scalar-measure-norm} as an analogous definition will play an important role in the vector-valued case.

Consider a single neuron $\phi_{v,\vec{w}}(\vec{x}) \coloneqq v \sigma(\vec{w}^\T\aug{\vec{x}})$, $\vec{x} \in \R^d$, where $v \in \R$ and $\vec{w} \in \Sph^{d}$.  In this scenario, it is clear that the $\inf$ in \cref{eq:SV-norm} is achieved by the one scaled Dirac measure $v \delta_{\vec{w}}$ (since any other combination of dictionary elements that represent a single neuron would have a larger norm). Thus, if we have the shallow neural network 
\begin{equation}
    \vec{x} \mapsto \sum_{k=1}^K v_k \sigma(\vec{w}_k^\T\aug{\vec{x}}), \quad \vec{x} \in \R^d,
\end{equation}
where the input weights $\vec{w}_k \in \Sph^{d}$ are all unique,
it is known~\citep[cf.,][p.~6]{bach2017breaking} that the $\inf$ in \eqref{eq:SV-norm} is achieved by
$\sum_{k=1}^K v_k \delta_{\vec{w}_k}$.
Therefore,
\begin{equation}
    \norm*{\vec{x} \mapsto \sum_{k=1}^K v_k \sigma(\vec{w}_k^\T\aug{\vec{x}})}_{\V_\sigma(\R^d)} = \norm*{\sum_{k=1}^K v_k  \delta_{\vec{w}_k}}_{\M(\Sph^d)} = \sum_{k=1}^K \norm{v_k \delta_{\vec{w}_k}}_{\M(\Sph^d)} = \sum_{k=1}^K \abs{v_k}, 
    \label{eq:scalar-path-norm}
\end{equation}
where the second equality uses the fact that the Dirac measures have disjoint support and the third equality follows from the property that $\norm{a \delta_{\vec{u}}}_{\M(\Sph^d)} = \abs{a}$, where $a \in \R$ and $\vec{u} \in \Sph^d$. 
The final quantity on the right above is the $\ell^1$-norm of all the output weights, which is precisely the regularization term in \eqref{opt:path-norm-sphere}.
Furthermore, if the activation function is homogeneous, then, by the NBT (\cref{thm:neural_balance}), the regularization of this quantity is equivalent to weight-decay regularization. Therefore, training a scalar-output shallow neural network with weight decay penalizes the variation norm of the network. 

\begin{remark} \label[remark]{rem:bias}
    To consider the unregularized bias scenario, the integral in \cref{eq:scalar-convex-NN} would instead take the form of an integral combination of neurons of the form $\vec{x} \mapsto \sigma(\vec{w}^\T\vec{x} + b)$, $(\vec{w}, b) \in \Sph^{d-1} \times \R$, against the measure $\nu \in \M(\Sph^{d-1} \times \R)$. The precise details are discussed by \citet[Appendix B]{ongie2019function} in the case when $\sigma$ is the ReLU. In that case, the corresponding variation space has an analytic description via the Radon transform~\citep{ongie2019function,parhi2021banach}. In this paper, we focus on the regularized bias scenario for notational convenience. As discussed in \cref{sec:nbt}, all results presented in this paper hold in the unregularized bias scenario.
\end{remark}

\subsection{Vector-Valued Variation Spaces} \label{sec:vv-spaces}
The vector-valued variation space is the set of functions defined analogously to the scalar-valued variation spaces:
\begin{equation}
    \V_{\sigma}(\R^d; \R^D) \coloneqq \curly*{f(\vec{x}) = \int_{\Sph^d} \sigma(\vec{w}^\T\aug{\vec{x}}) \dd \vec{\nu}(\vec{w}) \st \vec{x} \in \R^d, \vec{\nu} \in \M(\Sph^d; \R^D)},
\end{equation}
where $\vec{\nu} = (\nu_1, \ldots, \nu_D)$ is now a vector-valued measure (which takes values in $\R^D$ as opposed to $\R$) and plays the role of the output weight vector of each neuron.
Analogous to \cref{eq:scalar-measure-norm}, define the total variation norm of a measure $\vec{\nu}$ as
\begin{equation}
    \norm{\vec{\nu}}_{2, \M} \coloneqq \sup_{\substack{\Sph^d = \bigcup_{i=1}^n A_i \\ n \in \mathbb{N}}} \sum_{i=1}^n \norm{\vec{\nu}(A_i)}_2 = \sup_{\substack{\Sph^d = \bigcup_{i=1}^n A_i \\ n \in \mathbb{N}}} \sum_{i=1}^n \paren*{\sum_{j=1}^D \abs{\nu_j(A_i)}^2}^{1/2}. \label{eq:this-is-the-norm}
\end{equation}
The choice of norm in the above display is a common choice for the total variation norm of a vector-valued measure. Furthermore, $(\M(\Sph^d; \R^D), \norm{\dummy}_{2, \M})$ is a Banach space. We refer the reader to the monograph of \citet{diestel1977vector} for a full treatment of vector-valued measures and the accompanying results. This leads to a norm on functions of the form 
\begin{equation}
    f_\vec{\nu}(\vec{x}) \coloneqq \int_{\Sph^d} \sigma(\vec{w}^\T\aug{\vec{x}}) \dd\vec{\nu}(\vec{w}), \quad \vec{x} \in \R^d,
    \label{eq:V-convex-NN}
\end{equation}
as
\begin{equation}
    \norm{f}_{\V_{\sigma}(\R^d;\R^D)} \coloneqq \inf_{\substack{\vec{\nu} \in \M(\Sph^d; \R^D) \\ f = f_\vec{\nu}}} \norm{\vec{\nu}}_{2, \M}.
    \label{eq:VV-norm}
\end{equation}

To connect back to the optimization \cref{opt:path-norm-sphere}, consider a single vector-valued neuron $\phi_{\vec{v}, \vec{w}}(\vec{x}) \coloneqq \vec{v} \sigma(\vec{w}^\T\aug{\vec{x}})$, $\vec{x} \in \R^d$, where $\vec{v} \in \R^D$ and $\vec{w} \in \Sph^{d}$. 
As in the scalar-valued scenario, 
the $\inf$ is achieved by the measure $\vec{v} \delta_{\vec{w}}$. This is a vector (in $\R^D$) multiplied by a scalar-valued Dirac measure and is therefore a vector-valued measure in $\M(\Sph^d; \R^D)$. Thus, as in the scalar-valued case, for the shallow vector-valued neural network
\begin{equation}
    \vec{x} \mapsto \sum_{k=1}^K \vec{v}_k
    \sigma(\vec{w}_k^\T\aug{\vec{x}}), \quad \vec{x} \in \R^d,
    \label{eq:V-shallow-NN}
\end{equation}
where the input weights $\vec{w}_k \in \Sph^{d}$ are all unique, the $\inf$ in \eqref{eq:VV-norm} is achieved by
\begin{equation}
    \sum_{k=1}^K \vec{v}_k \delta_{\vec{w}_k}.
\end{equation}
Writing $\vec{v}_k = (v_{k, 1}, \ldots, v_{k, D})$, a calculation reveals that
\begin{align}
   \norm*{\vec{x} \mapsto \sum_{k=1}^K \vec{v}_k\sigma(\vec{w}_k^\T\aug{\vec{x}})}_{\V_\sigma(\R^d;\R^{D})}
   = \norm*{\sum_{k=1}^K \vec{v}_k \delta_{\vec{w}_k}}_{2, \M}
   = \sum_{k=1}^K \norm*{\vec{v}_k \delta_{\vec{w}_k}}_{2, \M}
   = \sum_{k=1}^K \norm{\vec{v}_k}_2, \label{eq:2-M-norm}
\end{align}
where we used the property that $\norm{\vec{a} \delta_{\vec{u}}}_{2, \M} = \norm{\vec{a}}_2$, where $\vec{a} \in \R^D$ and $\vec{u} \in \Sph^d$ \citep[cf.,][Section~4.2.3]{BoyerRepresenter}. 
From \cref{opt:path-norm-sphere}, we see immediately see that this choice of norm on $\M(\Sph^d; \R^D)$ corresponds to weight-decay regularization when $\sigma$ is homogeneous.

We remark that several other norms have been previously proposed for vector-valued networks/measures \citep[see][]{parhi2022kinds,korolev2022two}. These prior works essentially treat each output separately. This type of norm is fundamentally different than the one proposed in \cref{eq:VV-norm}. Furthermore, these other norms do not correspond to weight-decay regularization. These different norms and their relationships are discussed in \cref{app:equivalent-VV-M-norms}.

\subsubsection{The Curse of Dimensionality}
The space $\V_{\sigma}(\R^d; \R^D)$ has intriguing approximation properties, which carry over from the scalar-valued case (which are known). Let $Q^d \coloneqq [0,1]^d$ denote the unit cube in $\R^d$ and define
\begin{equation}
    \V_{\sigma}(Q^d; \R^D) \coloneqq \curly{f: Q^d \to \R^D \st \text{there exists } g \in \V_{\sigma}(\R^d; \R^D) \text{ such that } g|_{Q^d} = f}.
\end{equation}
This space is a Banach space when equipped with the norm
\begin{equation}
    \norm{f}_{\V_{\sigma}(Q^d; \R^D)} \coloneqq \inf_{\substack{g \in \V_{\sigma}(\R^d; \R^D) \\ g|_{Q^d} = f}} \norm{g}_{\V_{\sigma}(\R^d; \R^D)}.
\end{equation}
By restricting our attention to a bounded domain, we have the continuous embedding $\V_{\sigma}(Q^d; \R^D) \subset L^2(\R^d; \R^D)$. For each $f = (f_1, \ldots, f_D) \in \V_{\sigma}(Q^d; \R^D)$, we have, for $j = 1, \ldots, D$, that $f_j \in \V_{\sigma}(Q^d)$ (the scalar-valued variation space restricted to $Q^d$).

In the scalar-valued case, the Maurey--Jones--Barron lemma~\citep{MaureyPisier,JonesGreedy,BarronUniversal} says that, given $f_j \in \V_{\sigma}(Q^d)$, there exists a $K$-term approximant
\begin{equation}
    f_j^K(\vec{x}) = \sum_{k=1}^K v_{k, j} \sigma(\vec{w}_{k, j}^\T \aug{\vec{x}})
    \label{eq:scalar-approximant}
\end{equation}
with $v_{k, j} \in \R$ and $\vec{w}_{k, j} \in \Sph^d$ such that
\begin{equation}
    \norm{f_j - f_j^K}_{L^2(Q^d)} \leq C_0 C_\sigma \norm{f_j}_{\V_{\sigma}(Q^d)} K^{-1/2},
    \label{eq:MJB}
\end{equation}
where $C_0 > 0$ is an absolute constant independent of $d$ and
\begin{equation}
    C_\sigma = \sup_{\vec{w} \in \Sph^d} \norm{\vec{x} \mapsto \sigma(\vec{w}^\T\aug{\vec{x}})}_{L^2(Q^d)}.
\end{equation}
This result is remarkable since it establishes that, for any function in $\V_{\sigma}(Q^d)$, there exists an approximant whose error decays at a rate independent of the input dimension $d$. Here, we note that the constant $C_\sigma$ depends on (essentially) the volume the domain $Q^d$ since $\sigma$ is continuous, and is therefore also independent of the input dimension $d$ (since the volume of $Q^d$ is $1$). This result has a straightforward extension to the vector-valued case. This is summarized in the following theorem whose proof can be found in \cref{app:approx}, which shows that any function in $\V_{\sigma}(Q^d; \R^D)$ can be approximated in $L^2$ by a network of width $K$ with an error that decays at a rate $K^{-1/2}$, independent of the input and output dimensions $d$ and $D$.

\begin{theorem} \label{thm:approx}
    Given $f \in \V_{\sigma}(Q^d; \R^D)$, there exists a $K$-term approximant of the form
    \begin{equation}
        f_K(\vec{x}) = \sum_{k=1}^{K} \vec{v}_k \sigma(\vec{w}_k^\T\aug{\vec{x}}), \quad \vec{x} \in \R^d,
    \end{equation}
    with $\vec{v}_k \in \R^D$ and $\vec{w}_k \in \Sph^d$ such that
    \begin{equation}
    \norm{f - f_K}_{L^2(Q^d; \R^D)} \leq C_0 C_\sigma D^\frac{3}{2} \norm{f}_{\V_{\sigma}(Q^d; \R^D)} K^{-1/2},
    \end{equation}
    where $C_0$ and $C_\sigma$ are as above and the $L^2(Q^d; \R^D)$-norm is specified by
    \begin{equation}
        \norm{f}_{L^2(Q^d; \R^D)} = \paren*{\int_{Q^d} \norm{f(\vec{x})}_2^2 \dd\vec{x}}^{1/2}.
    \end{equation}
\end{theorem}

\begin{remark}
    \Cref{thm:approx} sets the stage for the investigation of dimension-free nonlinear minimax rates of estimation for multi-output neural networks. These results have recently been carried out for single-output neural networks trained with weight decay by \citet{parhi2022near}.
\end{remark}

\section{Representer Theorem for Vector-Valued Variation Spaces}\label{sec:vv_rep_thm}

The discussion above has shown that finite-width neural networks are effective at approximating functions in vector-valued variation spaces.  This section considers the problem of fitting data with functions in $\V_{\sigma}(\R^d; \R^D).$ The main result here is a \emph{representer theorem} that shows that finite-width neural networks are solutions to such problems and bounds the (sufficient) widths of networks in terms of the number of data points. This has an important implication: The infinite-dimensional learning problem can be solved by training a finite-width neural network, and increasing the width beyond the given bound will not yield a smaller objective value.

\begin{theorem} \label{thm:VV-rep}
    Let $(\vec{x}_1, \vec{y}_1), \ldots, (\vec{x}_N, \vec{y}_N) \in \R^d \times \R^D$ be a finite dataset. Then, there exists a solution to the variational problem
    \begin{equation}
        \inf_{f \in \V_{\sigma}(\R^d; \R^D)} \: \sum_{i=1}^N \mathcal{L}(\vec{y}_i, f(\vec{x}_i)) + \lambda \norm{f}_{\V_{\sigma}(\R^d; \R^D)}, \quad \lambda > 0,
        \label{eq:VV-opt}
    \end{equation}
    where the loss function $\mathcal{L}(\dummy, \dummy)$ is lower semicontinuous in its second argument, which takes the form
    \begin{equation}
        f^\star(\vec{x}) = \sum_{k=1}^{K_0} \vec{v}_k \sigma(\vec{w}_k^\T\aug{\vec{x}}), \quad \vec{x} \in \R^d, \label{eq:sparse-net}
    \end{equation}
    where $K_0 \leq \min\{N^2, ND\}$. Here, $\vec{v}_k \in \R^D$ and $\vec{w}_k \in \Sph^d$.
\end{theorem}
The proof of \cref{thm:VV-rep} appears in \cref{app:VV-rep}. What is remarkable here is the bound $K_0 \leq N^2$. Indeed, for large $D$, this bound improves the bound of $ND + 1$ predicted by Carath\'eodory's theorem~\citep{bredies2023extreme}. Note that \cref{thm:VV-rep} applies to a variation space based on any continuous activation function $\sigma$. Furthermore, we mention again that the result also holds in the unregularized bias scenario upon the appropriate modifications discussed in \cref{rem:bias}. If $\sigma$ is homogeneous, then the regularization is equivalent to weight decay. The result of \cref{thm:VV-rep} also applies to the entire solution set to \cref{eq:VV-opt}. Indeed, the solution set to \cref{eq:VV-opt} is nonempty, convex, weak$^*$ compact, and its extreme points take the form of sparse vector-valued networks as in \cref{eq:sparse-net}.

\begin{corollary}\label[corollary]{cor:V-weight-decay-VV-sol}
    Let
    $\mathcal{L}(\dummy, \dummy)$ be lower semicontinuous in its second argument. Moreover, let $\sigma:\R \to \R$ be any homogeneous activation function. Then, any solution to the neural network training problem
    \begin{equation}
        \min_{\{\vec{w}_k, \vec{v}_k\}_{k=1}^K} \sum_{i=1}^N \mathcal{L}\left(\vec{y}_i, \sum_{k=1}^{K}\vec{v}_k \sigma(\vec{w}^{T}_k\aug{\vec{x}}_i)\right) + \frac{\lambda}{2} \sum_{k=1}^K \norm{\vec{v}_k}_2^2 + \norm{\vec{w}_k}_2^2, \quad \lambda > 0,
        \label{eq:shallow-weight-decay-opt}
    \end{equation}
    is a solution to the variational problem \cref{eq:VV-opt}, so long as $K \geq \min\{N^2, ND\}$.
\end{corollary}
\begin{proof}
    By \cref{thm:VV-rep}, there always exists a solution to \cref{eq:VV-opt} that takes the form of a shallow vector-valued neural network with less than $\min\{N^2, ND\}$ neurons.  
    Thus, a solution to \cref{eq:VV-opt} must exist in the space of all shallow vector-valued neural networks with $K \geq \min\{N^2, ND\}$ neurons.
    By \eqref{eq:2-M-norm}, any solution to  
    \begin{equation}
        \min_{\{\vec{w}_k, \vec{v}_k\}^{K}_{k=1}} \sum_{i=1}^N \mathcal{L}\left(\vec{y}_i, \sum_{k=1}^{K}\vec{v}_k \sigma(\vec{w}^{T}_k\aug{\vec{x}}_i)\right) + \lambda \sum_{k=1}^K \norm{\vec{v}_k}_2, 
        \quad\mathrm{s.t.}\quad \|\vec{w}_k\|_2=1, \: k=1,\dots,K,
    \end{equation}
    is a solution to \eqref{eq:VV-opt}.
    The result then follows by the equivalence between the problem in the above display with \cref{eq:shallow-weight-decay-opt} as discussed in \cref{sec:nbt}.
\end{proof}

\begin{remark}
    The vector-valued variation space $(\V_{\sigma}(\R^d; \R^D), \norm{\dummy}_{\V_{\sigma}(\R^d; \R^D)})$ is an example of a reproducing kernel Banach space (RKBS)~\citep{zhang2009reproducing,lin2022reproducing}. Indeed, this can be readily deduced from the fact that the scalar-valued variation spaces are reproducing kernel Banach spaces~\citep{BartolucciRKBS,SpekRKBS}.
\end{remark}

\subsection{A Representer Theorem for Deep Neural Networks}
\Cref{thm:VV-rep} can be extended to DNNs by using the techniques developed by \citet[Theorem~3.2]{parhi2022kinds}. The extension is summarized in \cref{thm:deep-VV-rep} and the proof can be found in \cref{app:deep-VV-rep}.

\begin{theorem}\label{thm:deep-VV-rep}
    Let $(\vec{x}_1, \vec{y}_1), \ldots, (\vec{x}_N, \vec{y}_N) \in \R^{d_0} \times \R^{d_L}$ be a finite dataset. Then, there exists a solution to the variational problem
    \begin{equation}
        \inf_{\substack{f^{(1)}, \cdots, f^{(L)} \\ f^{(\ell)} \in \V_{\sigma}(\R^{d_{\ell-1}}; \R^{d_\ell})}} \: \sum_{i=1}^N \mathcal{L}(\vec{y}_i, f^{(L)} \circ \cdots \circ f^{(1)}(\vec{x}_i)) + \lambda \sum^{L}_{\ell=1} \norm{f^{(\ell)}}_{\V_{\sigma}(\R^{d_{\ell-1}}; \R^{d_{\ell}})}, \quad \lambda > 0,
        \label{eq:deep-VV-opt}
    \end{equation}
    where the loss function $\mathcal{L}(\dummy, \dummy)$ is lower semicontinuous in its second argument, which takes the form
    \begin{align}\label{eq:deep-VV-form}
         f^\star(\vec{x}) = \mathbf{A}^{(L)} \circ \vec{\sigma} \circ \mathbf{A}^{(L-1)} \circ \cdots \circ \vec{\sigma} \circ \mathbf{A}^{(1)}(\vec{x}) \quad \vec{x} \in \R^d,
    \end{align}
    where, for each layer $\ell = 1,\dots,L$, the function $\mathbf{A}^{(\ell)} (\bm{z}) = \mathbf{V}^{(\ell)}\vec{z} - \vec{b}^{(\ell)}$ is an affine mapping with weight matrix $\mathbf{V}^{(\ell)} \in \R^{d_{\ell} \times d_{\ell-1}}$ and bias vector $\bm{b}^{(\ell)} \in \R^{d_{\ell}}$, where $\vec{\sigma}$ applies the activation function $\sigma: \R \to \R$ component-wise.
\end{theorem}
\section{Neuron Sharing in Neural Network Solutions}\label{sec:neuron-sharing}

This section describes a remarkable ``neuron sharing"  property of solutions to the weight decay optimization objective (\ref{opt:general_weight_decay}) and the variational problem (\ref{eq:VV-opt}). This refers to the fact that each neuron in a solution is encouraged to contribute to every output, as opposed to networks in which different neurons are used to represent different output functions. This indicates how different outputs may influence each other in the training process. It also provides a new viewpoint for multi-task learning with neural networks: When the activation function is homogeneous, weight decay encourages the learning of features that are useful for multiple tasks/outputs. 

The neuron sharing property arises from the definition of the $\V_{\sigma}(\R^d; \R^D)$-norm in \cref{eq:VV-norm}. In particular,
the $\norm{\dummy}_{\V_{\sigma}(\R^d; \R^D)}$-norm regularized problem
\begin{align}\label{opt:vv-norm-reg}
  \min_{f \in \V_{\sigma}(\R^d; \R^D)} \paren*{\mathcal{J}(f) \coloneqq \sum_{i=1}^{N}
  \mathcal{L}(\vec{y}_i,f(\vx_i)) + \lambda \|f\|_{\V_{\sigma}(\R^d; \R^D)}}, \quad
  \lambda > 0,
\end{align}
where $(\vec{x}_1,\vec{y}_1),\dots,(\vec{x}_N,\vec{y}_N) \in \R^d \times \R^D$ is any fixed dataset and $\mathcal{L}(\dummy, \dummy)$ is lower semicontinuous in its second argument, favors solutions that share neurons. We quantify this explicitly in \cref{thm:neuron-sharing}.

\begin{theorem}\label{thm:neuron-sharing}
    Let $f$ be a finite-width vector-valued neural network with unique input weights of the form
    \begin{align}
        f(\vec{x}) = \sum_{k=1}^{K} \vec{v}_k \sigma(\vec{w}^{T}_k\aug{\vec{x}}), \quad \vec{x} \in \R^d,
    \end{align}
    with $\|\vec{w}_k\|_2 = 1$, $k=1,\dots,K$. Then, there exists $
    \delta > 0$ (that depends on $\lambda$ and the data) such that, if $\|\vec{w}_1 - \vec{w}_2 \|_2 < \delta$ and these two neurons contribute exclusively to two disjoint subsets of the outputs, then the neural network
    \begin{align}\label{eq:f-hat-ns-proof}
           \hat{f}(\vx) = f(\vec{x}) - \vec{v}_1 \sigma(\vec{w}^{T}_1 \aug{\vec{x}}) + \vec{v}_1 \sigma(\vec{w}^{T}_2 \aug{\vec{x}}),
    \end{align}
    which shares one neuron across both sets of outputs has a strictly smaller objective value, i.e., $\mathcal{J}(\hat{f}) < \mathcal{J}(f)$. 
\end{theorem}

\begin{proof}
    Since the input weights are all unit norm, \Cref{opt:vv-norm-reg} reduces to
    \begin{align}
        \mathcal{J}(f) = \sum_{i=1}^{N} \mathcal{L}(\vy_i, f(\vx_i)) + \lambda \sum_{k=1}^{K} \|\vec{v}_k\|_2.
    \end{align}
    Without loss of generality, suppose that the two neurons whose input weights are $\vec{w}_1$ and $\vec{w}_2$ contribute to two disjoint subsets of the $D$ outputs. In particular, the first neuron contributes to the outputs in index set $\mathcal{I}_{1} \subset \{1,\dots,D\}$, while the second contributes to outputs in the index set $\mathcal{I}_{2} \subset \{1,\dots,D\}$, where $\mathcal{I}_1 \cap \mathcal{I}_2 = \varnothing$. That is, $\vec{v}_1$ and $\vec{v}_2$ have disjoint support.
  Define
  \begin{align}
      \epsilon \coloneqq \lambda (\left \|\vec{v}_1\|_2 +  \|\vec{v}_2\|_2 - \|\vec{v}_1 + \vec{v}_2\|_2 \right) > 0.
  \end{align}
  Since the loss $\mathcal{L}$ is lower semicontinuous in its second argument, there exists $\gamma_i > 0$,  such that, if $\|\hat{f}(\vx_i) - f(\vx_i)\|_2 < \gamma_i$, then
  \begin{align}
      \mathcal{L}(\vy_i, \hat{f}(\vx_i)) - \mathcal{L}(\vy_i, f(\vx_i)) < \epsilon/N,
  \end{align}
  where we note that $\gamma_i$ depends on $\epsilon$, $N$, and $\vec{x}_i$.
  By \eqref{eq:f-hat-ns-proof}, we have that
  \begin{align}
      \|\hat{f}(\vx_i) - f(\vx_i)\|_2
      &= \|\vec{v}_1 \sigma(\vec{w}^{T}_1\aug{\vx}_i) - \vec{v}_1 \sigma(\vec{w}^{T}_2\aug{\vx}_i)\|_2 \nonumber \\ 
      &= \|\vec{v}_1\|_2 |\sigma(\vec{w}^{T}_1\aug{\vx}_i) - \sigma(\vec{w}^{T}_2\aug{\vx}_i)|.
  \end{align}
  The continuity of the activation function guarantees that there exists a $\delta_i > 0$ (that depends on $\gamma_i/\|\vec{v}_1\|_2$) such that
  \begin{align}
      \|\vec{w}_1 - \vec{w}_2\|_2 < \delta_i \implies |\sigma(\vec{w}^{T}_1\aug{\vx}_i) - \sigma(\vec{w}^{T}_2\aug{\vec{x}}_i)| < \gamma_i/\|\vec{v}_1\|_2.
  \end{align}
  Therefore, if $\|\vec{w}_1 - \vec{w}_2 \|_2 <\delta \coloneqq \min_{i=1, \ldots, N} \delta_i$, then $|\sigma(\vec{w}^{T}_1\aug{\vx}_i) - \sigma(\vec{w}^{T}_2\aug{\vec{x}}_i)| < \gamma_i/\|\vec{v}_1\|_2$, for any $i=1,\ldots,N$.
  Consequently, this implies that
  \begin{align}
      \left( \sum_{i=1}^{N} \mathcal{L}(\vy_i, \hat{f}(\vx_i)) - \mathcal{L}(\vy_i, f(\vx_i)) \right)
      < \epsilon
  \end{align}
  whenever $\|\vec{w}_1 - \vec{w}_2 \|_2 <\delta$.
  To complete the proof, observe that whenever $\|\vec{w}_1 - \vec{w}_2 \|_2 <\delta$, we have that
  \begin{align*}
      \mathcal{J}(\hat{f}) - \mathcal{J}(f) &=\left( \sum_{i=1}^{N} \mathcal{L}(\vy_i, \hat{f}(\vx_i)) - \mathcal{L}(\vy_i, f(\vx_i)) \right) 
      + \lambda \left(\|\hat{f}\|_{\V_{\sigma}(\R^{d};\R^{D})} - \|f\|_{\V_{\sigma}(\R^{d};\R^{D})}\right)\\
      &=\left( \sum_{i=1}^{N} \mathcal{L}(\vy_i, \hat{f}(\vx_i)) - \mathcal{L}(\vy_i, f(\vx_i)) \right)
      +\lambda\left( \|\vec{v}_1 + \vec{v}_2\|_2 - \left(\|\vec{v}_1\|_2 +  \|\vec{v}_2\|_2\right)\right)\\
      &< \epsilon + \lambda\left( \|\vec{v}_1 + \vec{v}_2\|_2 - (\|\vec{v}_1\|_2 +  \|\vec{v}_2\|_2)\right) \\
      &=0. \numberthis
  \end{align*}
  Thus, if two neurons have sufficiently close input weights, removing one and having the other one be shared strictly decreases the objective in \cref{opt:vv-norm-reg}.
  \end{proof}
  
    Note that by sharing neurons $\hat{f}$ always has a strictly smaller $\V_{\sigma}(\R^{d};\R^{D})$-norm than $f$. Indeed, we have that
  \begin{align}
      \norm{\hat{f}}_{\V_{\sigma}(\R^d; \R^D)} - \norm{f}_{\V_{\sigma}(\R^d; \R^D)} &= \|\vec{v}_1 + \vec{v}_2\|_2 - \left(\|\vec{v}_1\|_2 +  \|\vec{v}_2\|_2\right) < 0,
  \end{align}
  where the equality follows from the fact that the only neurons that are different between the two functions are the ones with input weights $\vec{w}_1$ and $\vec{w}_2$ and the inequality follows from the triangle inequality (which is strict since $\vec{v}_1$ and $\vec{v}_2$ have disjoint support).

\Cref{thm:neuron-sharing} along with the discussion in \Cref{sec:nbt} shows that, when the activation function $\sigma$ is homogeneous, vector-valued neural networks trained with weight decay are encouraged to share neurons.
Trained networks that exhibit neuron sharing are important in multi-task learning problems, e.g., multi-class classification, where components of the labels could have some relationships or correlations. Therefore, the neuron sharing phenomenon exhibited by solutions to weight decay regularized problems provides some explanation towards its efficacy when used for classification tasks. We illustrate the types of architectures that are favored by weight decay regularization in \cref{fig:neuron_sharing}. We also verify this numerically in \cref{subsec:sharing}.

\begin{figure}[htb]
    \centering
    \includegraphics[width=\textwidth]{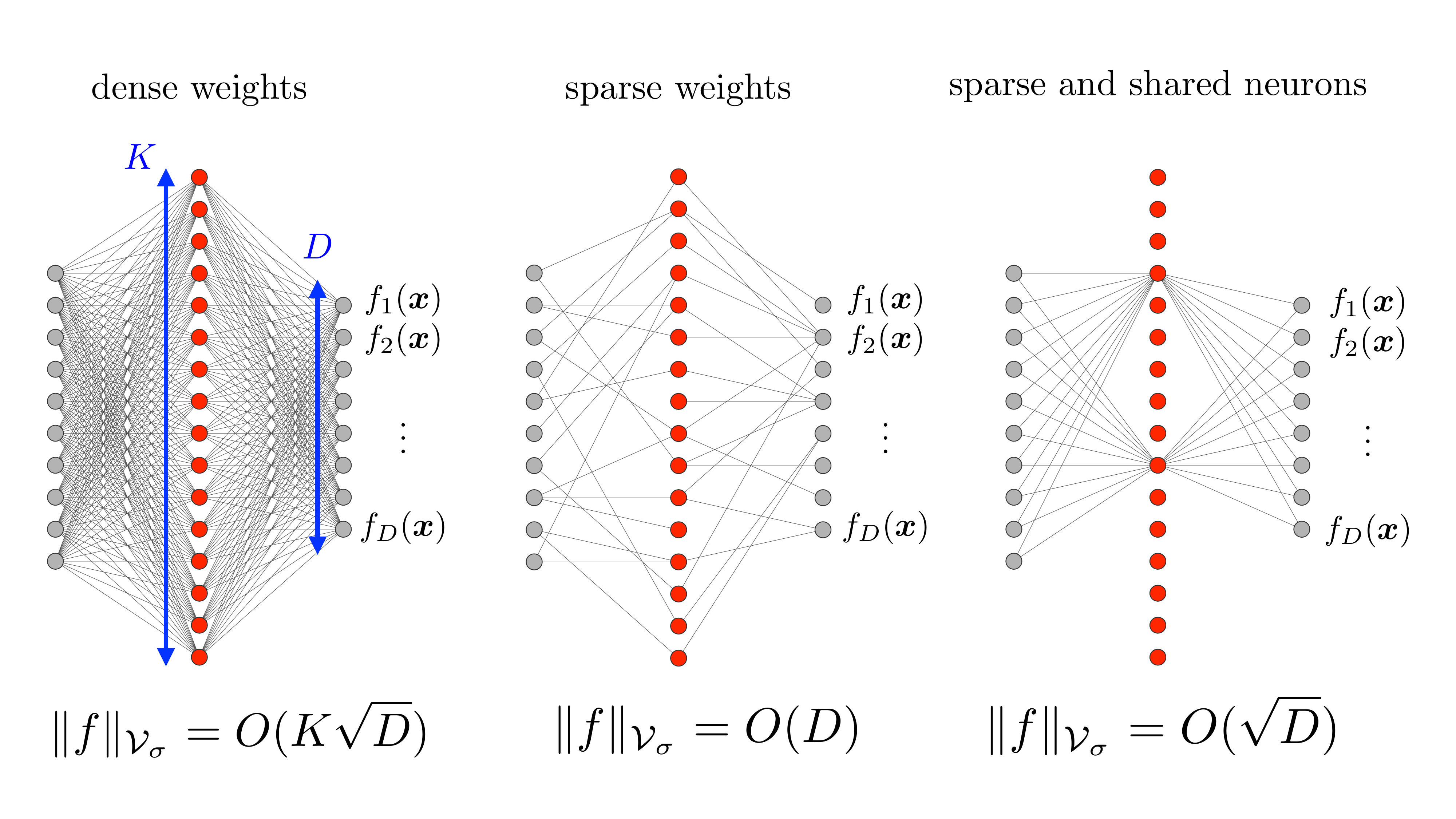}
    \caption{
    Three neural networks with different weight-sparsity patterns. The input weights are normalized to lie on the sphere and the components of the output weights are all $O(1)$. In the case of homogeneous activation functions, weight decay minimizes the $\V_{\sigma}(\R^d; \R^D)$ norm and therefore favors the right-most architecture. This architecture exhibits both neuron sparsity and neuron sharing. Each output depends on the same few neurons. This observation also gives insight into the regularity of the optimal functions: They favor functions that only vary in a few directions across all outputs. This is in contrast with the middle network where each output has variation in a small number of directions, but this set of directions can be different for each output.}
    \label{fig:neuron_sharing}
\end{figure}

\section{Data-Dependent Width Bounds and DNN Compression} \label{sec:bounds}
In this section, we complement our previous analysis of the vector-valued variation spaces. We use the NBT (\cref{thm:neural_balance}) to refine the number of neurons predicted by the representer theorems (\cref{thm:VV-rep,thm:deep-VV-rep}). In particular, these new bounds are \emph{data dependent} and improve previous results that have been reported in the literature~\citep[e.g.,][Proposition~7]{jacot2022feature}. Furthermore, these bounds are applicable to all layers with homogeneous activation functions within \emph{any} DNN.

We use the NBT, to relate the weights that minimize the weight decay objective with minimizers of a \emph{convex} multi-task lasso problem. This convex problem is data dependent and is based on the learned representations of the training data. We use this reduction to derive bounds on the widths of DNNs trained with weight decay (\cref{thm:widthbound}). A by-product of our analysis are, to the best of our knowledge, the first sparsity bounds for the (convex) multitask lasso problem (\cref{thm:multi_task_lasso_bound}). This result may be of independent interest. This investigation also motivates the proposal of a new principled and computationally efficient procedure to compress pre-trained DNNs (\cref{sec:compressing-dnns}).

\subsection{Width Bounds}\label{sec:width-bounds}
Let $f_{\vec{\theta}}$ be a DNN  that minimizes \eqref{opt:general_weight_decay}. Consider any layer of this DNN that has homogeneous activation functions and suppose that it has a width of $K$ neurons. Let $\vec{\varphi}_i$ denote the learned feature representation of the $i$th training example $\vec{x}_i$ that is input to this layer. Furthermore, let $\curly{\vec{w}_k}_{k=1}^K$ and $\curly{\vec{v}_k}_{k=1}^K$ denote the input and output weights of this layer, respectively. Then, the \emph{output features} learned by this layer are given by
\begin{equation}
    \vec{\psi}_i \coloneqq \sum_{k=1}^K \vec{v}_k \, \sigma(\vec{w}_k^\T\aug{\vec{\varphi}_i}), \quad i = 1, \ldots, N,
\end{equation}
where we recall that $\aug{\vec{\varphi}_i}$ augments a $1$ to $\vec{\varphi}_i$ to account for a bias term. Thanks to the homogeneity of $\sigma$, we can write
\begin{equation}
    \vec{\psi}_i \coloneqq \sum_{k=1}^K \tilde{\vec{v}}_k \,\sigma(\tilde{\vec{w}}_k^\T\aug{\vec{\varphi}_i}), \quad i = 1, \ldots, N, \label{eq:output-f}
\end{equation}
where $\tilde{\vec{w}}_k \coloneqq \vec{w}_k / \norm{\vec{w}_k}_2$ and $\tilde{\vec{v}}_k \coloneqq \vec{v}_k \norm{\vec{w}_k}_2$. Let $\phi_{k, i} \coloneqq \sigma(\tilde{\vec{w}}_k^\T\aug{\vec{\varphi}_i})$. Consequently, let 
\begin{equation}
    \vec{\phi}_i = (\phi_{1, i}, \ldots, \phi_{K, i}) \in \R^K, \quad i = 1, \ldots, N, \label{eq:post-activation-f}
\end{equation}
denote the \emph{post-activation} features.

A corollary to the representer theorem (\cref{thm:VV-rep}) is that there exists an optimal representation of this layer with $\leq N^2$ neurons. This is often a loose upper bound, especially when the learned features are highly structured. In \cref{thm:widthbound} we present \emph{data-dependent} bounds on the widths of DNN layers based on the intrinsic dimensions (ranks) of the learned data representations.
\begin{theorem}
\label{thm:widthbound}
    Let $(\vec{x}_1, \vec{y}_1), \ldots, (\vec{x}_N, \vec{y}_N)$ be any finite dataset and let $\mathcal{L}(\dummy, \dummy)$ be any loss function that is lower semicontinuous in its second argument. Furthermore, let $f_{\vec{\theta}}$ be a DNN that solves
    \begin{equation}\label{opt:gen-weight-decay-main-result}
    \min_{\vec{\theta}} \:\sum_{i=1}^{N} \mathcal{L}(\vy_i, f_{\vec{\theta}}(\vx_i)) + \frac{\lambda}{2} \|\vec{\theta}\|^2_{2}.
    \end{equation}
    Consider any layer of $f_{\vec{\theta}}$ with homogeneous activation functions and let $\curly{\vec{\phi}_i}_{i=1}^N$ and $\curly{\vec{\psi}_i}_{i=1}^N$ denote the learned post-activation and output features defined according to \cref{eq:post-activation-f} and \cref{eq:output-f}, respectively. If
    \begin{align}
        r_\mat{\Phi} &= \dim \spn\curly{\vec{\phi}_i}_{i=1}^N \nonumber \\
        r_\mat{\Psi} &= \dim \spn\curly{\vec{\psi}_i}_{i=1}^N
    \end{align}
    denotes the dimensions of the subspaces spanned by the learned features, then there exists an equivalent representation of the DNN (in the sense that it still minimizes \cref{opt:gen-weight-decay-main-result}) where this layer has at most $r_\mat{\Phi} \, r_\mat{\Psi} \leq N^2$ neurons. Furthermore, the equivalent representation can be found by solving a convex multi-task lasso problem.
\end{theorem}

\begin{proof}
Given the learned post-activation and output features $\curly{\vec{\phi}_i}_{i=1}^N$ and $\curly{\vec{\psi}_i}_{i=1}^N$ of the training examples $\curly{\vec{x}_i}_{i=1}^N$ from one layer of the DNN $f_\vec{\theta}$ with homogeneous activation functions, observe that, by the NBT (\cref{thm:neural_balance}), $\curly{\tilde{\vec{v}}_k}_{k=1}^K$ (as defined in \cref{eq:output-f}) must minimize
\begin{equation}\label{opt:path_norm_constrained}
        \min_{\{{\vec{v}}_k\}^{K}_{k=1}} \: \sum_{k=1}^{K}\| {\vv}_k \|_{2} \quad\mathrm{s.t.}\quad \vec{\psi}_i = \sum_{k=1}^{K} \vec{v}_k \phi_{k,i}, \: i = 1,\cdots,N.
\end{equation}
Therefore, if we replace $\curly{\tilde{\vec{v}}_k}_{k=1}^K$ with any solution to \cref{opt:path_norm_constrained} and rebalance the weights, the new DNN would still minimize \cref{opt:gen-weight-decay-main-result}.

To complete the proof, we observe that \cref{opt:path_norm_constrained} is a convex multi-task lasso problem. As we shall prove in \cref{thm:multi_task_lasso_bound} in \cref{sec:multi-task-lasso} below, there exists a solution to that multi-task lasso problem with at most $r_\mat{\Phi} \, r_\mat{\Psi}$ nonzero vectors. By a dimension-counting argument we always have the bound $r_\mat{\Phi} \, r_\mat{\Psi} \leq N^2$. Finally, observe that we can always find the compressed representation by solving the convex multi-task lasso problem. Thanks to the NBT, upon rebalancing the weights into and out of this layer, the compressed DNN still minimizes \cref{opt:gen-weight-decay-main-result}.
\end{proof}

Note that this result is not restricted to only one layer in a DNN. Indeed, it can be applied in a layer-wise manner. Thus, this result provides bounds on the widths of any layer in a DNN with homogeneous activation functions. This result also improved previous results that have appeared in the literature since it is data dependent. Furthermore, our generic bound of $N^2$ also improves the recent bound of $N(N+1)$ of \citet[Proposition~7]{jacot2022feature}.

Our data-dependent bounds are particularly relevant since it has been observed empirically that the subspaces spanned by learned features are often low-dimensional~\citep{nar2019crossentropy, waleffe2020principal,feng2022rank, huh2023low, kwon2024efficient, yaras2024compressible}. These empirical observations have also been backed by theoretical arguments~\citep{papyan2020prevalence,le2022training}. Therefore, we see that \cref{thm:widthbound} motivates the design of a principled and computationally efficient procedure to compress pre-trained DNNs. This procedure proceeds in a layer-by-layer manner and solves the convex multi-task lasso problem \cref{opt:path_norm_constrained}. We empirically evaluate the performance of this compression procedure in \cref{sec:compressing-dnns} on various architectures.

\subsection{Sparsity of Solutions to the Multi-Task Lasso Problem}\label{sec:multi-task-lasso}
The main ingredients of the proof of \cref{thm:widthbound} were (i) a reduction to the convex multi-task lasso problem (ii) the invocation of the sparsity bounds on the multi-task lasso problem, which we derive in this section. We consider the multi-task lasso problem as formulated by \citet{obozinski2006multi,obozinski2010joint,argyriou2008convex, obozinski2008high}. This problem has also been studied in the context of inverse problems with multiple measurement vectors~\citep{cotter2005sparse,chen2006theoretical,sun2009efficient,bajwa2015conditioning}. Our result on the sparsity of multi-task lasso minimizers appears in \cref{thm:multi_task_lasso_bound}. This result is new, to the best of our knowledge, and may be of independent interest. The proof can be found in \cref{app:multi_task_lasso}. 

\begin{theorem}\label{thm:multi_task_lasso_bound}
Consider the multi-task lasso problem
\begin{equation}
    \min_{\mathbf{V}=[\vv_1,\cdots,\vv_K]} \:\sum_{k=1}^{K} \| \vv_{k} \|_{2} \quad\mathrm{s.t.}\quad \mathbf{\Psi} = \mathbf{V} \mathbf{\Phi}, \label{opt:multi-task-lasso}
\end{equation}
where $\mathbf{\Psi} \in \mathbb{R}^{D \times N}$ and $\mathbf{\Phi} \in \mathbb{R}^{K \times N}$ are matrices with ranks $r_{\mathbf{\Phi}}$ and $r_{\mathbf{\Psi}}$, respectively.  Assume that the row space of $\mathbf{\Psi}$ is contained in the row space of $\mathbf{\Phi}$, which implies a solution exists. For any $K>0$ there exists a solution with at most $r_{\mathbf{\Phi}} r_{\mathbf{\Psi}} $ nonzero columns.  Furthermore, there exists $\mathbf{\Psi}$ such that no solution has fewer than $r_{\mathbf{\Phi}}$ nonzero columns.
\end{theorem}

Unlike the traditional lasso problem, which seeks a sparse solution in an unstructured manner, the multi-task lasso seeks a structured sparsity. In this setting, each column of $\mathbf{V}$ is associated with a block. This block will either be entirely zero or (typically) entirely nonzero. Since each column of $\mathbf{V}$ corresponds to a feature, this implies that each feature is used in either all prediction tasks or none. 
The proof utilizes a generalized version of Carathéodory's theorem along with the fact that for any solution $\mathbf{V}$, the dimension of the column space must be $r_{\mathbf{\Psi}}$.

The intuition behind the upper and lower bounds in \cref{thm:multi_task_lasso_bound} is as follows. If every row of $\mathbf{\Psi}$ can be synthesized using the same $r_{\mathbf{\Phi}}$ rows in $\mathbf{\Phi}$, while still minimizing the objective, then the lower bound is achieved.  On the other hand, if the minimum can be achieved by having $r_{\mathbf{\Psi}}$ linearly independent rows in $\mathbf{\Psi}$ synthesized by different subsets of $r_{\mathbf{\Phi}}$ rows in $\mathbf{\Phi}$, then the upper bound is met. Our numerical experiments in \cref{sec:gl_exps} demonstrate that the minimum number of nonzero columns in any solution may range between the upper and lower bounds depending on precise structure of $\mathbf{\Psi}$ and $\mathbf{\Phi}$.

\begin{remark}
    The lower bound follows from the following observation.  If a solution exists with fewer than $r_{\mathbf{\Phi}}$ nonzero columns, then all the rows of $\mathbf{\Psi}$ belong one of the (finitely many) lower dimensional subspaces in the row space of $\mathbf{\Phi}$ spanned by certain subsets its rows. But collectively, these subspaces do not contain all points in the row space of $\mathbf{\Phi}$, and so there are infinitely many points that cannot be represented with fewer than $r_\mathbf{\Phi}$ rows of $\mathbf{\Phi}$.
\end{remark}


\begin{remark}
    Observe that when $D = 1$, \cref{opt:multi-task-lasso} reduces to the classical lasso problem. In that case, the sparsity of solutions to the classical lasso has been investigated by \citet{rosset2004boosting,tibshirani2013lasso}.
\end{remark}

\captionsetup[sub]{font=footnotesize}

\section{Experiments} \label{sec:exp}
In this section we present three numerical experiments that validate our theory and demonstrate its utility in practice. Our first experiment demonstrates that weight decay encourages neuron sharing. Our second experiment validates the bound presented for the multi-task lasso problem in \cref{thm:multi_task_lasso_bound}. Our third experiment compresses layers for pre-trained VGG-19~\citep{Simonyan2014VeryDC,wang2021pufferfish} and AlexNet~\citep{krizhevsky2017imagenet} models via the multi-task lasso convex optimization problem as in \cref{thm:widthbound}. In particular, we show that this principled compression approach preserves the training loss, accuracy, and weight decay objective of the model.\footnote{The code to reproduce our experiments can be found at \url{https://github.com/joeshenouda/vv-spaces-nn-width}. We follow the reproduciblity guidance of \cite{shenouda2023guide}.}

\subsection{Neuron Sharing Simulation} \label{subsec:sharing}
To demonstrate that weight decay encourages neuron sharing we train three vector-valued shallow ReLU neural networks to fit a synthetic two-dimensional dataset. The dataset consists of $50$ samples where the features are two-dimensional vectors drawn i.i.d.\ from a multivariate normal distribution. The labels are three-dimensional and generated by passing the feature vectors through a randomly initialized ReLU neural network with five neurons.

All networks were initialized with one hundred and fifty neurons. The first network was trained with weight decay regularization. The second network was trained with $\ell^1$ regularization. The third network was trained with no regularization. We used full batch gradient descent with the Adam optimizer. We used a learning rate of $2 \times 10^{-3}$ for two million iterations. The regularization parameter was $\lambda = 10^{-6}$ for weight decay and $\lambda = 10^{-9}$ for $\ell^1$-regularization. We chose $\lambda$ to be as large as possible such that the networks interpolate the data.

\begin{figure}
    \centering
    \begin{subfigure}[b]{\textwidth}
        \centering
        \includegraphics[width=\textwidth]{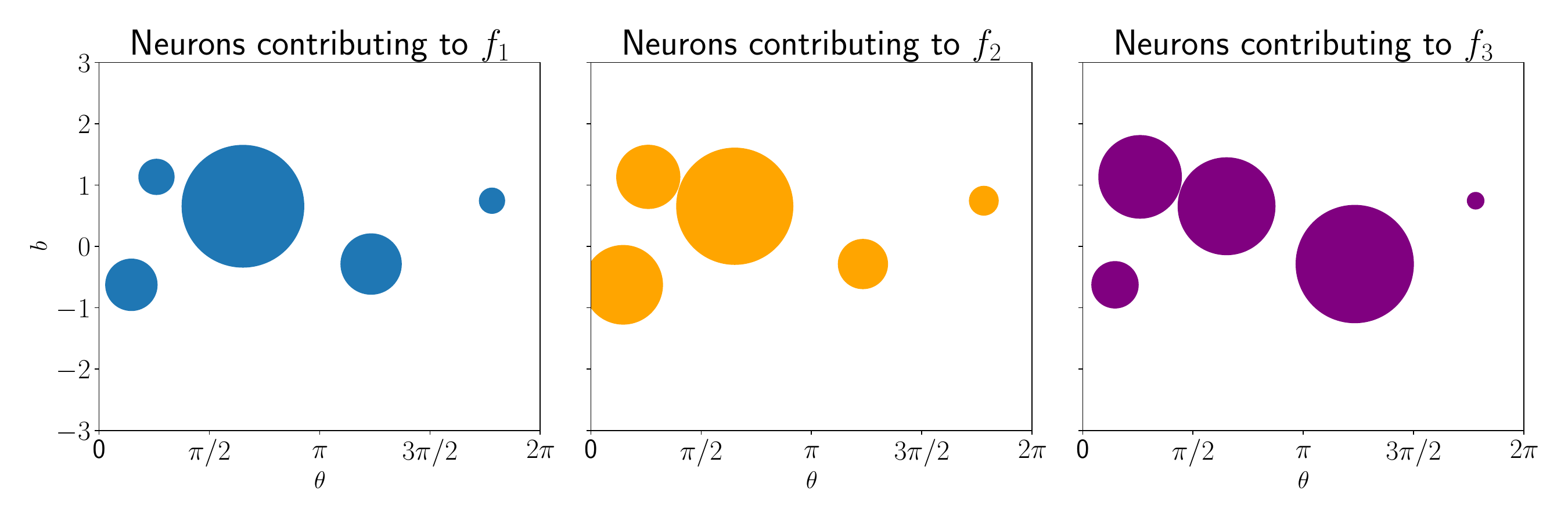}
        \caption{Weight Decay Regularization ($\#$Active Neurons: 5)}   
        \label{fig:neuron-loc-wd}
    \end{subfigure}
    \medskip
    \begin{subfigure}[b]{\textwidth}
        \centering
        \includegraphics[width=\textwidth]{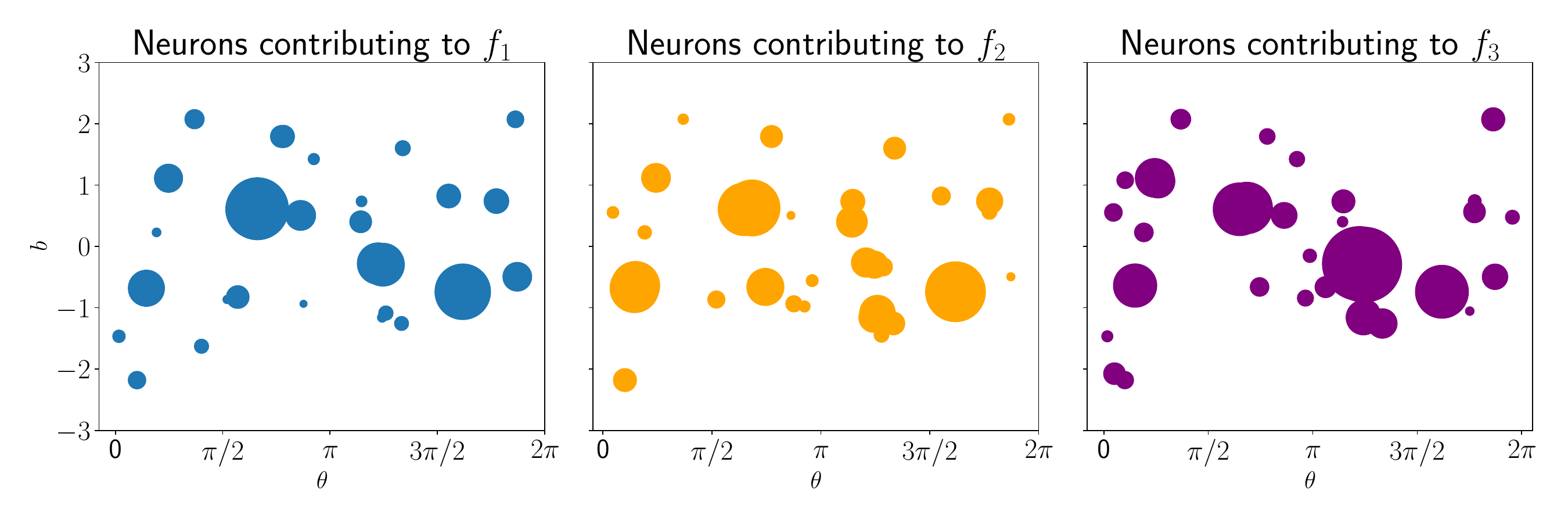}
        \caption{$\ell^1$-Regularization  ($\#$Active Neurons: 85)}   
        \label{fig:neuron-loc-L1}
    \end{subfigure}
    \medskip
    \begin{subfigure}[b]{\textwidth}
        \centering
        \includegraphics[width=\textwidth]{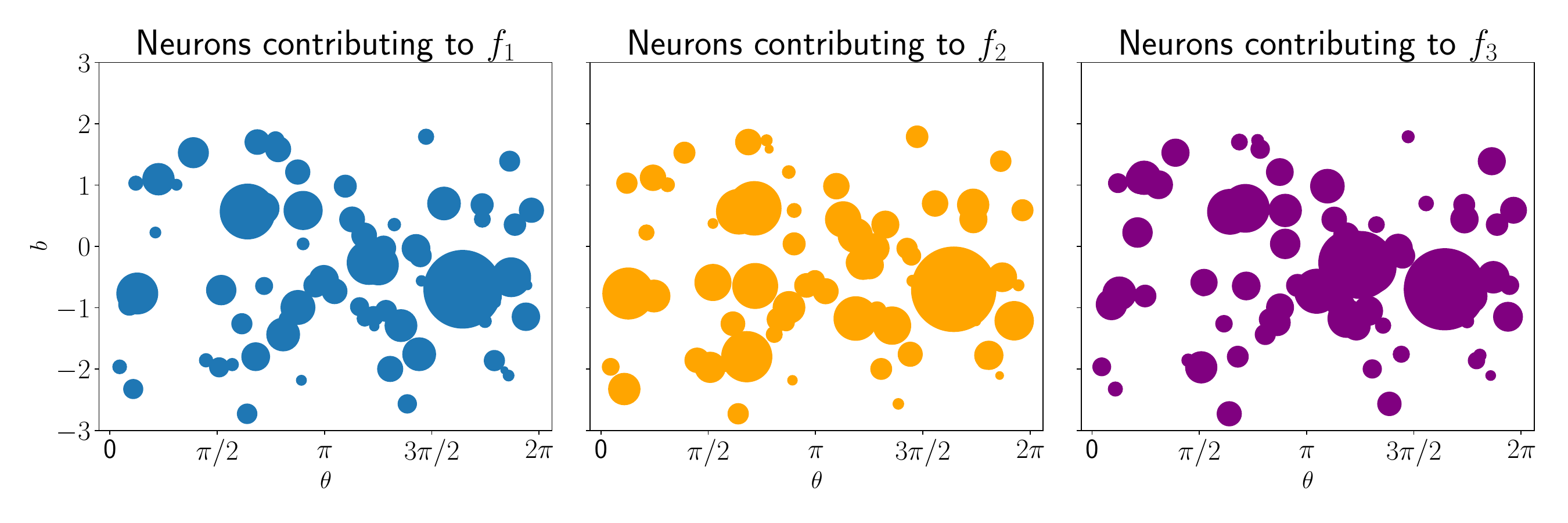}
        \caption{No Regularization  ($\#$Active Neurons: 130)}   
        \label{fig:neuron-loc-no-reg}
    \end{subfigure}
    \caption{We trained a three output two-dimensional ReLU neural network of the form $\vec{f}(x)=\sum_{k=1}^{K} \vv_k \sigma(\vw^{T}_{k}\vx+b_k)$  with weight decay, $\ell^1$-regularization, and no regularization. Let $f_1$, $f_2$ and $f_3$ denote the first, second, and third components of the outputs. We plot the locations of each active neurons under the $(\theta_k, b_k)$-parameterization. The size of the circles indicate the magnitude of the corresponding output weight vector. We see that in the case of weight decay, we have very few active neurons. Furthermore, those neurons that remain are shared across all outputs.}
    \label{fig:neuron_sharing_sim}
\end{figure}

In \cref{fig:neuron_sharing_sim} we plot the locations of the neurons contributing to each output of the trained network. Thanks to the homogeneity of the ReLU, after training, we normalize the input weights to be unit norm by absorbing the magnitude into the output weight. This reparameterization does not change the overall function mapping. This allows us to express each input weight as $\bm{w}_k = (\cos\theta_k, \sin\theta_k)$ with a bias $b_k \in \bR$. In other words, we can plot each neuron $\vec{x} \mapsto \sigma(\vec{w}_k^\T\vec{x} + b_k)$ with the two-dimensional coordinate $(\theta_k, b_k)$.

In \cref{fig:neuron_sharing_sim}, we show the $(\theta_k, b_k)$ pairs for every active neuron in the trained network. Our results show that, when regularizing with weight decay, the learned network is not only sparser (in terms of the number of active neurons), but also exhibits strong neuron sharing. In contrast with $\ell^1$-regularization there is not much neuron sparsity or neuron sharing. Finally, we see that no regularization results in a very dense network where all neurons are active. To generate the plots, we deem a neuron \emph{active} if the $\ell^1$-norm of the output weight is greater than $10^{-3}$. At the end of training we had $5$ neurons active for weight decay regularization, $85$ neurons active for $\ell^1$-regularization and $130$ neurons active with no regularization.

\subsection{Multi-Task Lasso Experiments}\label{sec:gl_exps}
We solve the multi-task lasso problem in \cref{thm:multi_task_lasso_bound} on randomly generated matrices $\mathbf{\Phi}$ and $\mathbf{\Psi}$ using CVXPy~\citep{diamond2016cvxpy}. Our experiments illustrate that, while our bounds hold, the exact number of nonzero columns depends on the data itself. In \cref{fig:dimensions_ML_lasso} we show histograms for the distribution of nonzero columns over 100 randomly generated pairs of $\mathbf{\Phi}$ and $\mathbf{\Psi}$. 

\begin{figure}[tb!]
    \centering
    \begin{subfigure}[b]{0.24\textwidth}
        \centering
        \includegraphics[width=\textwidth]{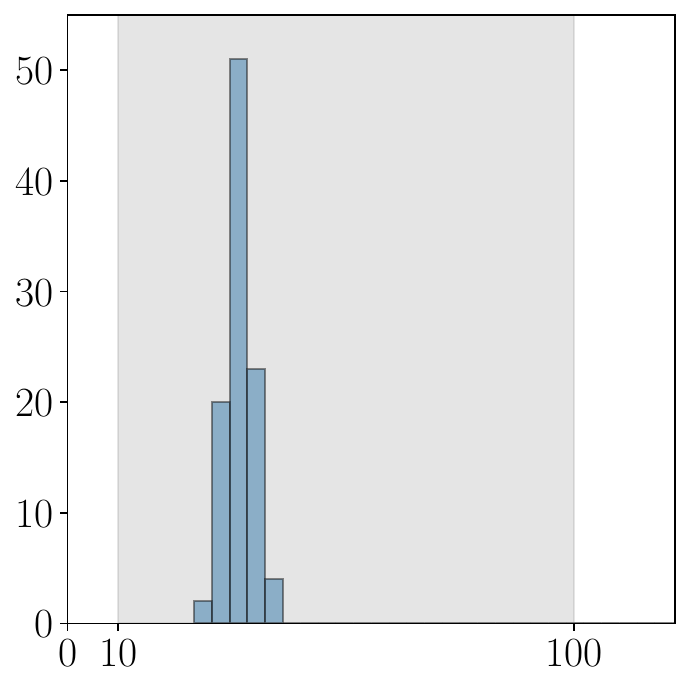}
        \caption{$D\!=\!10$,\!\! $N\!=\!10$,\!\! $K\!=\!500$}   
        \label{fig:f1-no-wd}
    \end{subfigure}
    \hfill
    \begin{subfigure}[b]{0.24\textwidth}
        \centering 
        \includegraphics[width=\textwidth]{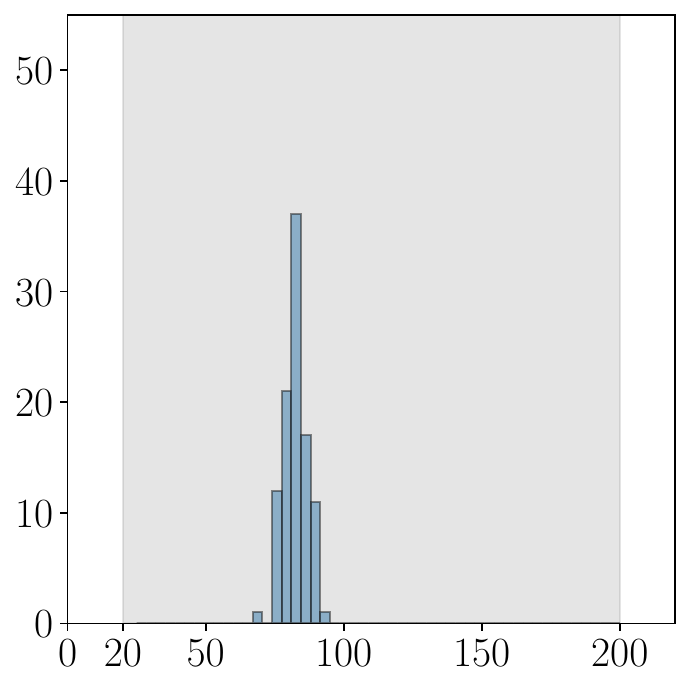}
        \caption{$D\!=\!10$,\!\! $N\!=\!20$,\!\! $K\!=\!500$}    
        \label{fig:f2-no-wd}
    \end{subfigure}
    \hfill
    \begin{subfigure}[b]{0.24\textwidth}
        \centering 
        \includegraphics[width=\textwidth]{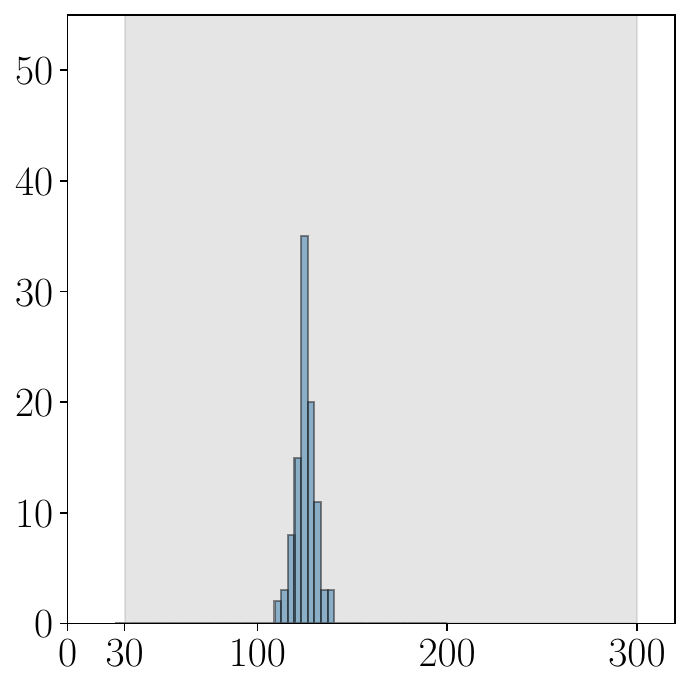}
        \caption{$D\!=\!10$,\!\! $N\!=\!30$,\!\! $K\!=\!500$}    
        \label{fig:f3-no-wd}
    \end{subfigure}
    \hfill
    \begin{subfigure}[b]{0.24\textwidth}
        \centering 
        \includegraphics[width=\textwidth]{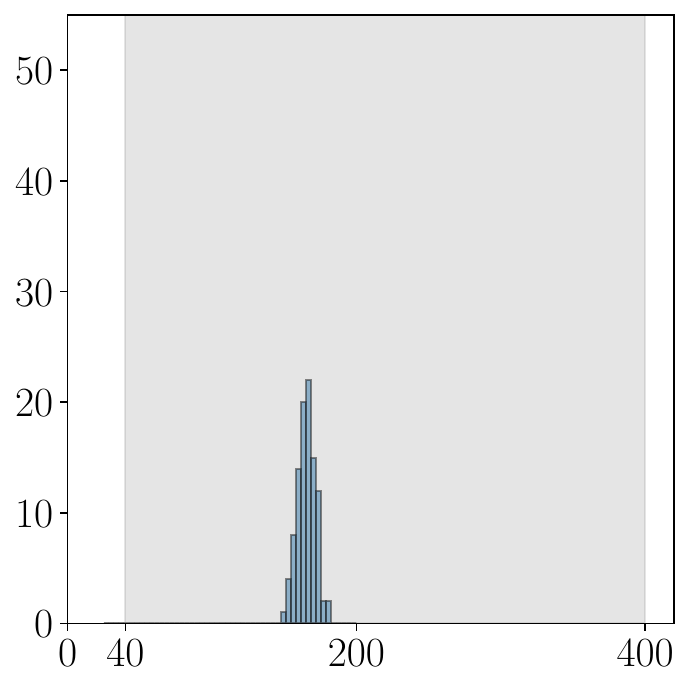}
        \caption{$D\!=\!10$,\!\! $N\!=\!40$,\!\! $K\!=\!500$}    
        \label{fig:f4-no-wd}
    \end{subfigure}
    \caption{
    Distribution of the number of active columns for the solutions to the multi-task lasso problem on randomly generated matrices of varying sizes. The horizontal axis is the number of nonzero columns in the optimal $\mathbf{V}$ and the vertical axis is the frequency. We ran this experiment for 100 randomly generated matrices. In all cases $r_{\mathbf{\Phi}}=N$ and $r_{\mathbf{\Psi}}=D$ so by~\cref{thm:multi_task_lasso_bound} we expect $N \leq \hat{K}\leq ND$. The shaded region indicates our theoretical bounds. The wide gap suggests that our upper bound can be sharpened.}
    \label{fig:dimensions_ML_lasso}
\end{figure}

In \cref{fig:ranks_ML_lasso} we perform a similar set of experiments but alter the underlying rank of $\mathbf{\Phi}$. This validates our bound showing that the sparsity of the solution can be much lower depending on the rank. We again see that the distribution is dependent on the data.  We note that, however, we never achieve our upper bound which may indicate that our bounds can be further sharpened.

\begin{figure}[tb!]
    \centering
    \begin{subfigure}[b]{0.24\textwidth}
        \centering
        \includegraphics[width=\textwidth]{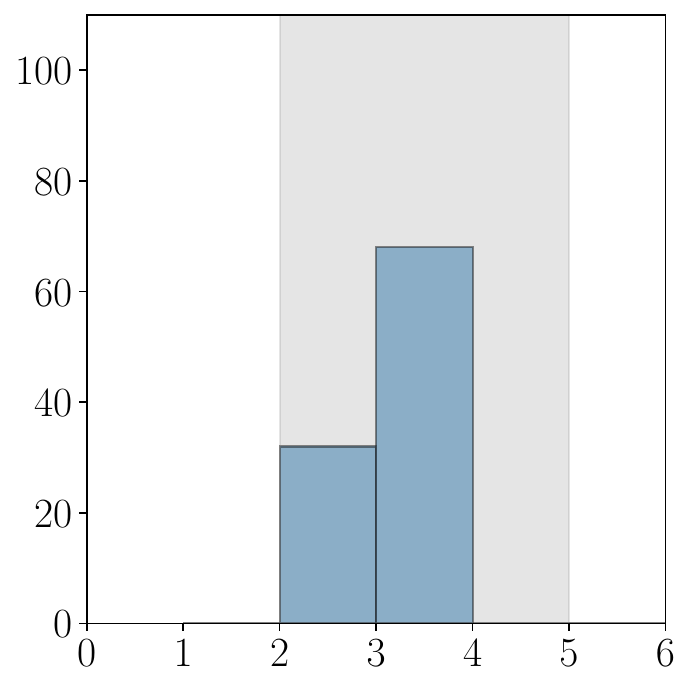}
        \caption{$r_{\mat{\Phi}} \!=\! 2$, $r_{\mat{\Psi}} \!=\! 2$}   
        \label{fig:d_10_n_20_k_200_r_2}
    \end{subfigure}
    \hfill
    \begin{subfigure}[b]{0.24\textwidth}
        \centering 
        \includegraphics[width=\textwidth]{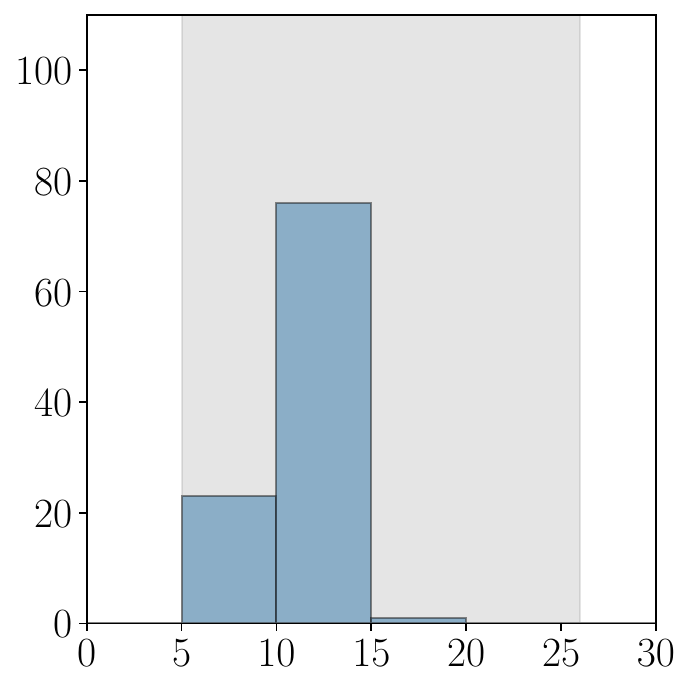}
        \caption{$r_{\mat{\Phi}} \!=\! 5$,$r_{\mat{\Psi}} \!=\! 5$}    
        \label{fig:d_10_n_20_k_200_r_5}
    \end{subfigure}
    \hfill
    \begin{subfigure}[b]{0.24\textwidth}
        \centering 
        \includegraphics[width=\textwidth]{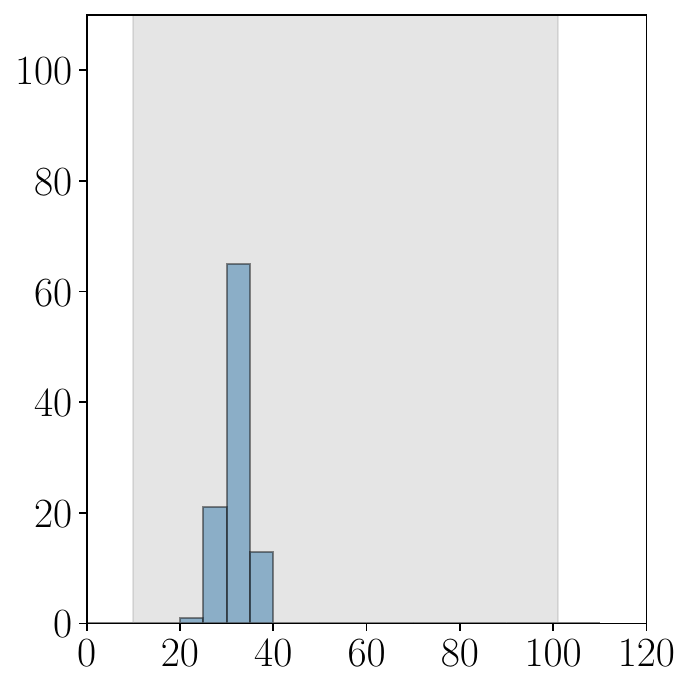}
        \caption{$r_{\mat{\Phi}}\!=\!10$, $r_{\mat{\Psi}}\!=\!10$}    
        \label{fig:d_10_n_20_k_200_r_10}
    \end{subfigure}
    \hfill
    \begin{subfigure}[b]{0.24\textwidth}
        \centering 
        \includegraphics[width=\textwidth]{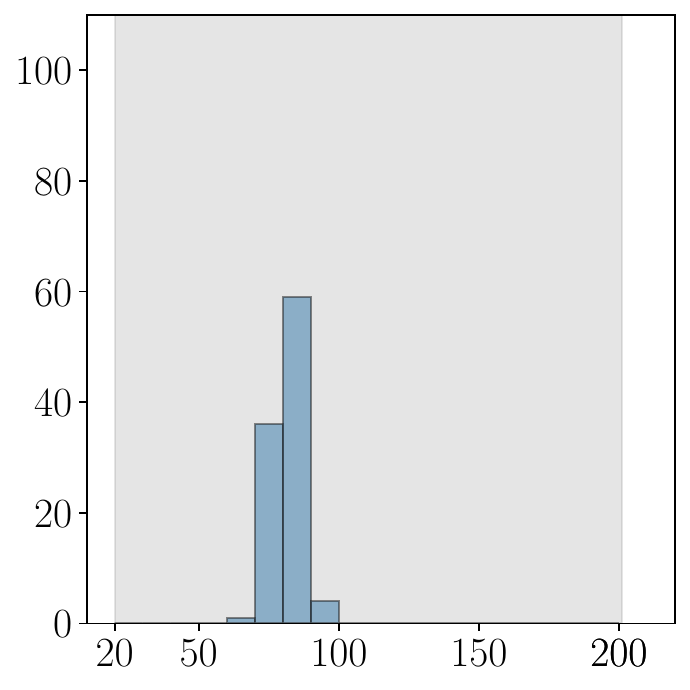}
        \caption{$r_{\mat{\Phi}}\!=\!20$, $r_{\mat{\Psi}}\!=\!10$}    
        \label{fig:d_10_n_20_k_200_r_20}
    \end{subfigure}
    \caption{Distribution of the number of active columns for the solutions to the multi-task lasso problem on randomly generated matrices with $\mathbf{\Phi}$ of various rank. The horizontal axis is the number of nonzero columns in the optimal $\mathbf{V}$ and the vertical axis is the frequency. We ran this experiment for 100 randomly generated matrices, in all cases $D=10, N=20$ and $K=200$. By~\cref{thm:multi_task_lasso_bound} we expect $r_{\mathbf{\Phi}} \leq \hat{K} \leq r_{\mathbf{\Phi}} \cdot r_{\mathbf{\Psi}}$. The shaded region indicates our theoretical bounds.}
    \label{fig:ranks_ML_lasso}
\end{figure}

To demonstrate that the sparsest solutions to the multi-task lasso problem \cref{thm:multi_task_lasso_bound} depends on the data matrices $\mathbf{\Phi}$ and $\mathbf{\Psi}$ we also ran a small-scale experiment similar to \cref{fig:ranks_ML_lasso} and \cref{fig:dimensions_ML_lasso}. However, in the next set of experiments, we exhaustively searched over all $2^{K}$ sparsity patterns that the solution may have. We arrive at the same conclusion as we did before: The sparsest solution can lie anywhere within our bound. This is depicted in \cref{fig:exhaustive_ML_lasso}.

\begin{figure}[tb!]
    \captionsetup[sub]{font=tiny}
    \centering
    \begin{subfigure}[b]{0.19\textwidth}
        \centering
        \includegraphics[width=\textwidth]{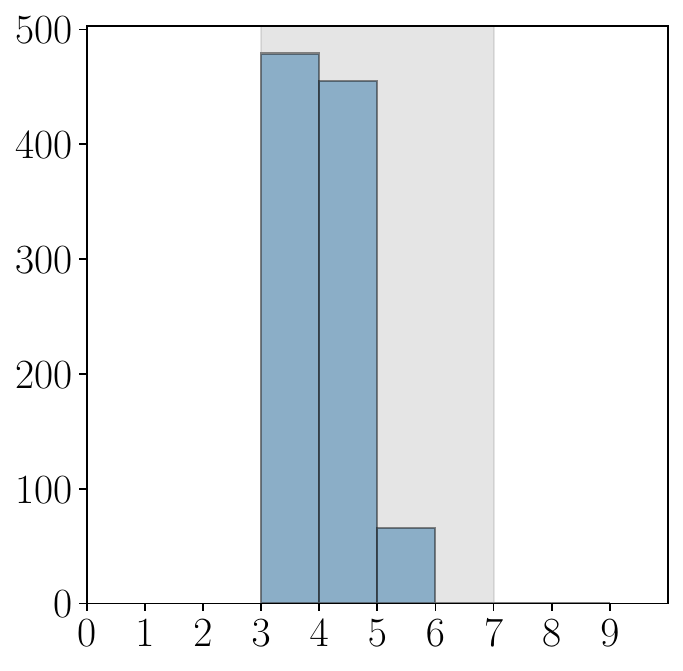}
        \caption{$D=2$, $N=3$, $K = 7$}   
        \label{fig:rank_K_7}
    \end{subfigure}
    \hfill
    \begin{subfigure}[b]{0.19\textwidth}
        \centering 
        \includegraphics[width=\textwidth]{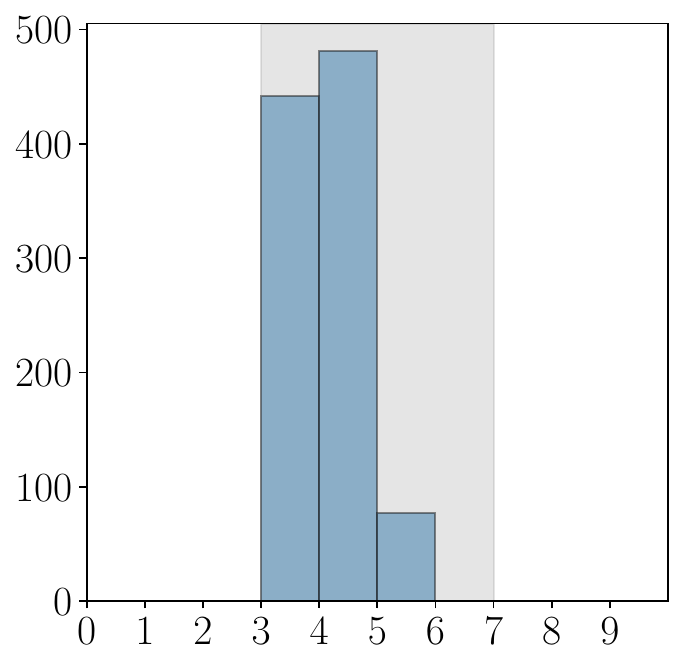}
        \caption{$D=2$, $N=3$, $K=8$}    
        \label{fig:rank_K_8}
    \end{subfigure}
    \hfill
    \begin{subfigure}[b]{0.19\textwidth}
        \centering 
        \includegraphics[width=\textwidth]{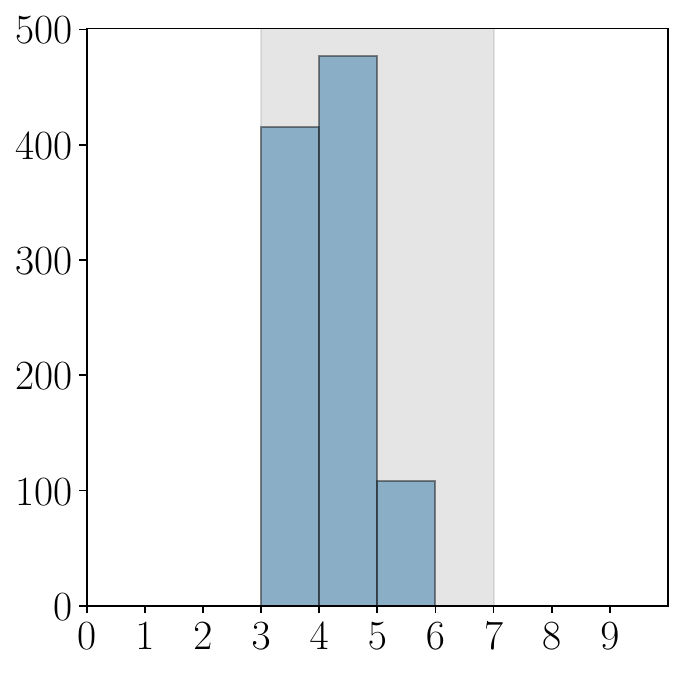}
        \caption{$D=2$, $N=3$, $K=9$}    
        \label{fig:rank_K_9}
    \end{subfigure}
    \hfill
    \begin{subfigure}[b]{0.19\textwidth}
        \centering 
        \includegraphics[width=\textwidth]{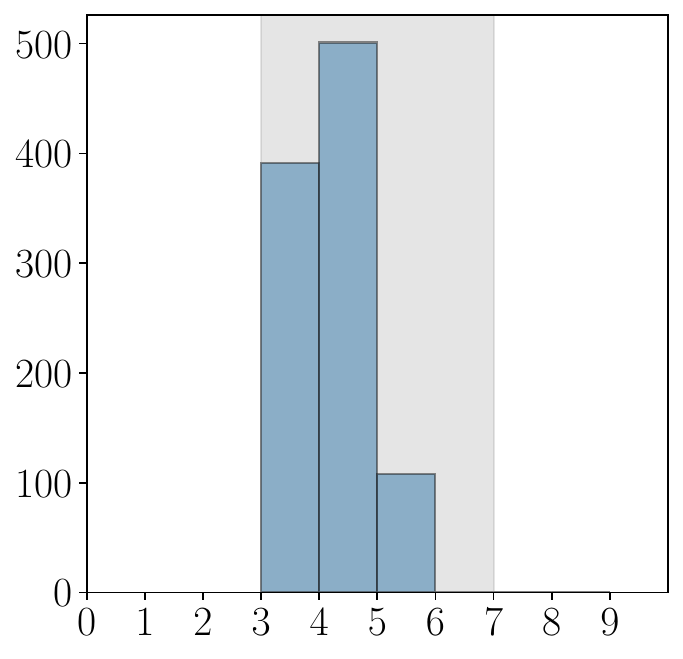}
        \caption{$D=2$,\! $N=3$,\! $K=10$}    
        \label{fig:rank_K_10}
    \end{subfigure}
    \hfill
    \begin{subfigure}[b]{0.19\textwidth}
        \centering 
        \includegraphics[width=\textwidth]{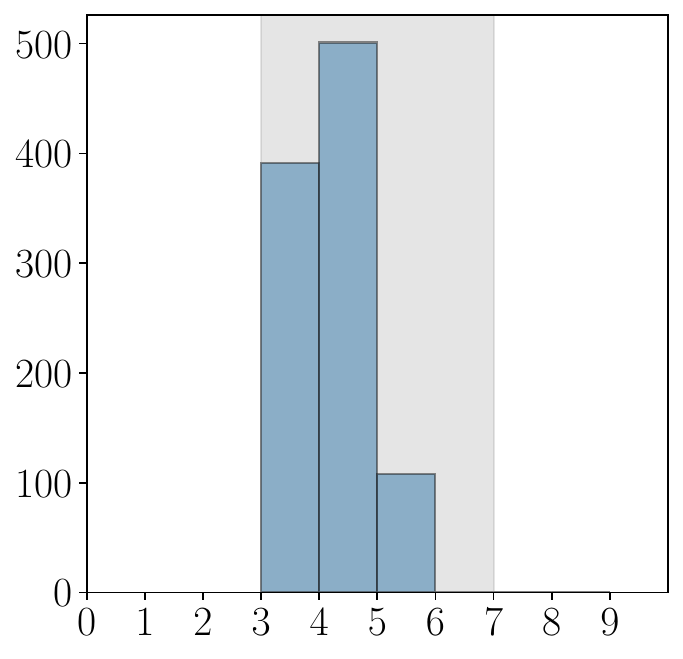}
        \caption{$D=2$, $N=3$, $K=11$}    
        \label{fig:rank_K_11}
    \end{subfigure}
    \caption{Distribution of number of nonzero columns for the solutions to the multi-task lasso problem. We ran this experiment over 1000 randomly generated matrices $\mathbf{\Phi}$ and $\mathbf{\Psi}$. The horizontal axis is the number of nonzero columns in the optimal $\mathbf{V}$ where each bin is left-inclusive corresponding to a single integer. The vertical axis is the frequency. We see that the sparsest solution is dependent on the data and can vary between our bounds. The area between the two shaded regions indicate our theoretical bounds. The gap indicates that our upper bound can be sharpened.}
    \label{fig:exhaustive_ML_lasso}
\end{figure}

\subsection{Compression of Pre-Trained DNNs}\label{sec:compressing-dnns}
We compress DNNs pre-trained on the CIFAR-10 dataset based on the principled approach outlined in \cref{sec:bounds}. The first model we consider is a pre-trained VGG-19 architecture trained with weight decay.\footnote{We use the pre-trained VGG-19 model from \url{https://github.com/hwang595/Pufferfish}.} This model consists of various convolutional and batch-norm layers followed by a fully-connected ReLU layer that contains 512 neurons. We run our compression procedure on this ReLU layer.

The output of this ReLU layer is the output of the entire network. There are $10$ outputs that correspond to the $10$ classes of CIFAR-10. Therefore, $\mathbf{\Psi} \in \mathbb{R}^{10 \times N}$ with $r_\mathbf{\Psi} \leq 10$. The post-activation feature matrix $\mathbf{\Phi} \in \R^{512 \times N}$ has an approximate rank of 10. We approximate the rank following the same procedure of \citet[Appendix D]{huh2023low}:  We threshold the singular values at a value of $10^{-3}$.

Thus, \cref{thm:widthbound} suggests that there exists an alternative optimal representation with no more than $100$ neurons in this layer that can be found by solving a convex multi-task lasso problem. We validate this by minimizing\footnote{We minimize the objective with proximal gradient methods.} a regularized version of the constrained multi-task lasso problem
\begin{equation}\label{opt:lagrangian}
    \min_{\{ \vv_k \}} \frac{1}{ND}\| \mathbf{V}\mathbf{\Phi} - \mathbf{\Psi}\|^{2}_{2} + \lambda \sum_{k=1}^{K} \| \vv_k \|_{2},
\end{equation}
with $\lambda = 3 \times 10^{-3}$.

The second model we consider is a pre-trained AlexNet architecture trained with weight decay.\footnote{We trained this model ourselves based on the implementation in \url{https://github.com/rasbt/deeplearning-models}.} We use this mode to illustrate how we can compress multiple layers in a DNN in a layer-by-layer fashion. This model has multiple convolutional layers followed by three fully connected ReLU layers with 9216, 4096, and 4096 neurons, respectively. We run our compression procedure on these ReLU layers.

For the first fully-connected layer we found that $r_{\mathbf{\Psi}} \approx 69$ and $r_{\mathbf{\Phi}} \approx 225$. For the second fully-connected layer we found that $r_{\mathbf{\Psi}} \approx 56$ and $r_{\Phi} \approx 69$. For the last fully-connected layer we found that $r_{\mathbf{\Psi}} \leq 10$ and $r_{\mathbf{\Phi}} \approx 56$. The approximate ranks were again computed with a threshold of $10^{-3}$ on the singular values. \Cref{thm:widthbound} suggestions that we can significantly compress the last two layers. For the first and second fully-connected layers we solve \cref{opt:lagrangian} with $\lambda = 10^{-6}$. For the last fully-connected layer we used $\lambda = 10^{-3}$.

The results for the compression of both VGG-19 and AlexNet can be found in~\cref{tab:cifar10}. Observe that there is almost no change in the training loss or the sum of squared weights of the model. The minor discrepancies between the original models and the compressed models are due to the fact that we solve the regularized problem \eqref{opt:lagrangian}.

\begin{table}[htb]
    \fontsize{30}{40}\selectfont
    \caption{Compression results for various layers of VGG-19 and AlexNet models pre-trained on the CIFAR-10 dataset. FC stands for fully connected.}
    \resizebox{1.\columnwidth}{!}{
    \begin{tabular}{lcccccc}
        \toprule
        & \multicolumn{2}{c}{Width} & \multicolumn{2}{c}{Train Loss (Test Acc.)} & \multicolumn{2}{c}{$\|\bm{\theta}\|^2_2$} \\
        \cmidrule(r){2-3}\cmidrule(r){4-5}\cmidrule(l){6-7}
        & Original & New  & Original & New & Original & New\\
        \midrule
        VGG-19 (last layer) &   512 & 99 & 0.0002\;(93.91\%) & 0.0002\;(93.90\%) & 3823.41 & 3816.70 \\
        AlexNet (last layer) & 4096 & 555 & 0.2203\;(72.17\%) & 0.2201\;(72.14\%) & 1141.50\; & 961.46\\
        AlexNet (penultimate layer) &  4096 & 2369 & 0.2203\;(72.17\%) & 0.2158\;(72.16\%) & 1141.50 & 1137.65\\
        AlexNet (first FC layer) &  9216 & 8208 & 0.2203\;(72.17\%) & 0.2190\;(72.18\%) & 1141.50 & 1137.33\\
        \cmidrule(r){1-1}
        AlexNet (last 3 layers) & (4096,4096,9216) & (555,2369,8208) & 0.2203\;(72.17\%) & 0.21521(72.13\%) & 1141.50 & 956.5\\
        \bottomrule
    \end{tabular}
    }
    \label{tab:cifar10}
\end{table}

\section{Related Work}
In this section, we highlight some related works and discuss how our work fits into the current literature.

\subsection{Variation Spaces and Representer Theorems for Neural Networks} Understanding neural networks by studying their associated variation space has been explored in the past \citep{BarronUniversal, kurkova2001bounds,mhaskar2004tractability,bach2017breaking}.  Representer theorems that show that finite-width shallow networks are solutions to data-fitting problems posed over these spaces have also been developed~\citep{parhi2021banach}. These results have been extended to a wide variety of activation functions~\citep{parhi2020role}, more general settings \citep{korolev2022two,BartolucciRKBS,SpekRKBS,wang2024sparse}, as well as deep networks~\citep{parhi2022kinds,bartolucci2024neural,wang2024hypothesis}. The variation norm has also shown practical benefits in learning sparse neural networks~\citep{yang2022better} and implicit neural representations~\citep{shenouda2024relus}.

It is important to note that these prior works primarily investigate scalar-output neural networks or simple Cartesian product constructions to account for multiple outputs. As we saw in \cref{sec:vv-spaces} (see also \cref{app:equivalent-VV-M-norms}), the vector-valued setting brings unique difficulties in the choice of the norm associated with the variation space, and many choices do not always coincide with regularizers used in practice such as weight decay. While we note that there has been work devoted to the investigation of vector-valued RKBSs for multi-task learning (see \citet{zhang2013vector} and references therein), we are the first to study and propose the $\V_{\sigma}(\R^d;\R^D)$ space and norm. Furthermore, our new representer theorem (\cref{thm:VV-rep}) improves many results in the representer theorem literature since the number of neurons does not depend on the output dimension (see \citet{bredies2023extreme} for more details on the output-dimension dependence).

\subsection{Weight Decay and Sparsity} Numerous works have observed that weight decay is biased towards solutions with fewer neurons ~\citep{savarese2019infinite,parhi2021banach, yang2022better}. Empirically,~\citet{yang2022better} have shown that training DNNs with weight decay for many iterations induces neuron-wise sparsity on the trained network. Furthermore,~\cite{parhi2022kinds,jacot2022feature} provide bounds on the widths for DNNs that solve the weight decay objective. Our bounds improve upon the ones presented in both of these works. In particular, we develop the first \emph{data-dependent} bounds. Moreover, our setting is also more general: Our theory holds for \emph{any} DNN architecture that has homogeneous activation functions.

Recently, the works of~\cite{ergen2021convex,Mishkin2022FastCO} have shown that training shallow ReLU neural networks with weight decay can be recast as a (very large) convex program. Their reformulation reveals that weight decay induces a sparsity-promoting regularizer. A similar observation was developed for the vector-valued case by \cite{sahiner2021vectoroutput}.

\subsection{Low-Rank Features}
It has been observed that DNNs are biased toward learning low-rank features. Theoretical evidence for this phenomenon has been developed by~\citet{du2018algorithmic,ji2018gradient,radhakrishnan2020alignment,le2022training,mousavi-hosseini2023neural}. This has also been observed empirically in~\citet{nar2019crossentropy,waleffe2020principal, feng2022rank, huh2023low, kwon2024efficient, yaras2024compressible}. We highlight that some works such as \citet{waleffe2020principal, kwon2024efficient} use this insight to develop specialized algorithms that compress models during training. In contrast, our theory and experiments indicate that narrower networks can be found regardless of the training procedure by solving a simple convex multi-task lasso. Moreover, \cref{thm:widthbound} guarantees that the compressed network has the same sum of squared weights and minimizes the same objective as the original network. 

\section{Conclusion}
In this work we proposed the $\V_{\sigma}(\R^{d};\R^{D})$ function space which gives insights in the inductive bias of vector-valued neural networks. This is a critical step towards understanding DNNs, the theory of which significantly lags behind that of shallow neural networks. We proved a new representer theorem for this space showing that finite-width vector-valued neural networks are solutions to the data-fitting problem over this space. For DNNs with homogeneous activation functions we developed a novel connection between training DNNs with weight decay and multi-task lasso to prove the sharpest known bounds on network widths. This result motivated the design of a principled and computationally efficient procedure to compress pre-trained DNNs. Finally, we presented some experimental results showing the validity and practical utility of our theory.


\acks{The authors thank Ryan Tibshirani for helpful discussions regarding the lasso and related methods, especially the use of Carathéodory's theorem for characterizing the sparsity of their solutions. JS was supported by the ONR MURI grant N00014-20-1-2787. RP was supported by the National Science Foundation (NSF)  Graduate  Research  Fellowship  Program  under  grant  DGE-1747503 while he was with the University of Wisconsin--Madison and the European Research Council under grant 101020573 while he was with the \'Ecole polytechnique f\'ed\'erale de Lausanne. KL was supported by the NSF grant DMS-2023239. RN was supported in part by the NSF grants DMS-2134140 and DMS-2023239, the ONR MURI grant N00014-20-1-2787, and the AFOSR/AFRL grant FA9550-18-1-0166.}

\appendix
\section{Different Norms on Vector-Valued Measures} \label{app:equivalent-VV-M-norms}
In this appendix, we will discuss different choices of norms one can equip the space of vector-valued measures $\M(\Omega; \R^D)$, where $\Omega$ is any locally compact Hausdorff space. Eventually, we will show in \cref{lemma:equivalent-VV-M-norms} that these choices are equivalent Banach norms, though only one choice (the one proposed in \cref{sec:vv-spaces}) corresponds to weight-decay regularization in neural networks with homogeneous activation functions.

By viewing $\M(\Omega; \R^D)$ as the $D$-fold Cartesian product $\bigtimes_{j=1}^D \M(\Omega)$, a na\"ive choice of norm would be the mixed norm
\begin{equation}
    \norm{\vec{\nu}}_{\M, p} \coloneqq \paren*{\sum_{j=1}^D \norm{\nu_j}_{\M(\Omega)}^p}^{1/p} = \paren*{\sum_{j=1}^D \paren*{\sup_{\substack{\Omega = \bigcup_{i=1}^n A_i \\ n \in \mathbb{N}}} \sum_{i=1}^n \abs{\nu_j(A_i)}}^p}^{1/p},
\end{equation}
with $p \geq 1$ and $\vec{\nu} = (\nu_1, \ldots, \nu_D)$ where each $\nu_j \in \M(\Omega)$, $j = 1, \ldots, D$. In this scenario, it is clear that $\M(\Omega; \R^D)$ is a Banach space when equipeed with the norm $\norm{\dummy}_{\M, p}$, $p \geq 1$. Alternatively, one could consider the norm
\begin{equation}
    \norm{\vec{\nu}}_{p, \M} \coloneqq \sup_{\substack{\Omega = \bigcup_{i=1}^n A_i \\ n \in \mathbb{N}}} \sum_{i=1}^n \norm{\vec{\nu}(A_i)}_p = \sup_{\substack{\Omega = \bigcup_{i=1}^n A_i \\ n \in \mathbb{N}}} \sum_{i=1}^n \paren*{\sum_{j=1}^D \abs{\nu_j(A_i)}^p}^{1/p},
\end{equation}
with $p \geq 1$. We have that $(\M(\Omega; \R^D), \norm{\dummy}_{p, \M})$ is also Banach space \citep[cf.][pg. 29]{diestel1977vector}. When $\Omega = \Sph^d$ and $p =2$, this coincides with \cref{eq:this-is-the-norm}. The next lemma shows that all of these norms are actually equivalent Banach norms.
\begin{lemma} \label[lemma]{lemma:equivalent-VV-M-norms}
Let $\Omega$ be any locally compact Hausdorff space. The norms $\norm{\dummy}_{p, \M}$, $p \geq 1$, and $\norm{\dummy}_{\M, q}$, $q \geq 1$ are all equivalent Banach norms for $\M(\Omega; \R^D)$.
\end{lemma}

\begin{proof}
    We first observe that  for any $p, q \geq 1$, the norms $\norm{\dummy}_{\M, p}$ and $\norm{\dummy}_{\M, q}$ are equivalent by the equivalence of $p$-norms in finite dimensions. Similarly, for any $p, q \geq 1$, the norms $\norm{\dummy}_{p, \M}$ and $\norm{\dummy}_{q, \M}$ are equivalent. Thus, it suffices to show that $\norm{\dummy}_{\M, p}$ and $\norm{\dummy}_{q, \M}$ are equivalent for some $p, q \geq 1$. We have for any $\vec{\nu} \in \M(\Omega; \R^D)$
    \begin{align*}
        \norm{\vec{\nu}}_{\M, 1}
        &= \sum_{j=1}^D \norm{\nu_j}_{\M(\Omega)} \\
        &= \sum_{j=1}^D \sup_{\substack{\Omega = \bigcup_{i=1}^n A_i \\ n \in \mathbb{N}}} \sum_{i=1}^n \abs{\nu_j(A_i)} \\
        &\geq \sup_{\substack{\Omega = \bigcup_{i=1}^n A_i \\ n \in \mathbb{N}}} \sum_{j=1}^D \sum_{i=1}^n \abs{\nu_j(A_i)} \\
        &= \sup_{\substack{\Omega = \bigcup_{i=1}^n A_i \\ n \in \mathbb{N}}} \sum_{i=1}^n \sum_{j=1}^D \abs{\nu_j(A_i)} \\
        &= \sup_{\substack{\Omega = \bigcup_{i=1}^n A_i \\ n \in \mathbb{N}}} \sum_{i=1}^n \norm{\vec{\nu}(A_i)}_1 = \norm{\vec{\nu}}_{1, \M}. \numberthis
    \end{align*}
    For the reverse inequality, observe that
    \begin{align*}
        \norm{\vec{\nu}}_{1, \M}
        &= \sup_{\substack{\Omega = \bigcup_{i=1}^n A_i \\ n \in \mathbb{N}}} \sum_{i=1}^n \sum_{j=1}^D \abs{\nu_j(A_i)} \\
        &\geq \sup_{\substack{\Omega = \bigcup_{i=1}^n A_i \\ n \in \mathbb{N}}} \sum_{i=1}^n \abs{\nu_{j_0}(A_i)} \\
        &= \norm{\nu_{j_0}}_{\M(\Omega)} \\
        &\geq \frac{1}{D} \sum_{j=1}^D \norm{\nu_j}_{\M(\Omega)} \\
        &= \frac{1}{D} \norm{\vec{\nu}}_{\M, 1}, \numberthis
    \end{align*}
    where $j_0 \coloneqq \argmax_{j \in [D]} \norm{\nu_j}_{\M(\Omega)}$.
    Therefore, we have that
    \begin{align}\label{eq:equi-norms}
        \norm{\vec{\nu}}_{1, \M} \leq \norm{\vec{\nu}}_{\M, 1} \leq D \norm{\vec{\nu}}_{1, \M},
    \end{align}
    which proves the lemma.
\end{proof}
\begin{remark}
    While the norms may be equivalent, only the $\norm{\dummy}_{2, \M}$-norm corresponds to weight-decay regularization as shown in \cref{sec:vv-spaces}. The use of some of these other norms have been explored in the literature. For example, \citet{parhi2022kinds} define a similar vector-valued variation space for ReLU networks with the $\norm{\dummy}_{\M, 1}$-norm, which does not correspond to weight-decay regularization.
\end{remark}

\subsection{Connection to the Total Variation of Function}
The total variation of a measure is different than the total variation of a function, but the ideas are tightly linked.
In the univariate case, consider a Radon measure $\mu \in \M(\R)$ and suppose there exists a function $g_\mu: \R \to \R$ such that
\begin{equation}
    \int_{\R} f(x) \dd \mu(x) = \int_{\R} f(x) \dd g_\mu(x),
\end{equation}
where $f$ is any bounded continuous function and the integral on the right-hand side is a Riemann--Stieltjes integral. If $g_\mu$ is differentiable and its derivative, denoted by $g_\mu'$, is in $L^1(\R)$, then we have the equality
\begin{equation}
    \norm{\mu}_{\M(\R)} = \norm{g_\mu'}_{L^1(\R)}.
\end{equation}
The quantity on the right-hand side is referred to as the \emph{total variation} of the function $g_\mu$. Furthermore, if $g_\mu$ is not differentiable in the classical sense, its distributional derivative can be identified with the Radon measure $\mu$, i.e., $g_\mu' = \mu$, where equality is understood in $\M(\R)$. In this case, $\norm{g_\mu'}_{\M(\R)} = \norm{\mu}_{\M(\R)}$ is the total variation of the function $g_\mu$.

This correspondence extends to higher dimensions. If the vector measure $\mu=\nabla g$ for some function $g:\R^d \to \R$, then the norm we are using $\norm{\mu}_{2, \M}=\norm{\nabla g}_{2, \M}$, is the \emph{isotropic total variation} of $g$. On the other hand, the norm $\norm{\mu}_{1, \M}=\norm{\nabla g}_{1, \M}$ corresponds to the \emph{anisotropic total variation} of $g$. These two notions of total variation in multiple dimensions are often used in image processing. 

The isotropic total variation is equivalently specified by in the dual form
\begin{equation}
    \TV_\iso(g) = \sup_{\substack{\varphi \in \mathcal{D}(\R^d; \R^d) \\ \norm{\varphi}_{L^\infty(\R^d; \R^d)} = 1}} \int_{\R^d} g(\vec{x}) \,\mathrm{div} \varphi(\vec{x}) \dd \vec{x},
\end{equation}
where
\begin{equation}
    \norm{\varphi}_{L^\infty(\R^d; \R^d)} \coloneqq \esssup_{\vec{x} \in \R^d} \: \norm{\varphi(\vec{x})}_2,
\end{equation}
and $\mathcal{D}(\R^d; \R^d)$ denotes the space of infinitely differentiable compactly supported functions mapping $\R^d \to \R^d$. We refer the reader to the book of \citet{evans2015measure} for more details about the total variation of a function.

\section{Proof of \Cref{thm:approx}} \label{app:approx}
\begin{proof}
    First notice that for any $f = (f_1, \ldots, f_D) \in L^2(\R^d; \R^D)$ we have that
    \begin{align*}
        \norm{f}_{L^2(Q^d; \R^D)}^2
        &= \int_{Q^d} \norm{f(\vec{x})}_2^2 \dd\vec{x} \\
        &= \int_{Q^d} \sum_{j=1}^D \abs{f_j(\vec{x})}^2 \dd\vec{x} \\
        &= \sum_{j=1}^D \int_{Q^d} \abs{f_j(\vec{x})}^2 \dd\vec{x} \\
        &= \sum_{j=1}^D \norm{f_j}_{L^2(Q^d)}^2. \numberthis
    \end{align*}
    For any $f \in \V_\sigma(Q^d; \R^D)$ we have that
    \begin{equation}
        \sum_{j=1}^D \norm{f_j}_{\V_\sigma(Q^d)}^2 \leq \paren*{\sum_{j=1}^D \norm{f_j}_{\V_\sigma(Q^d)}}^2.
    \end{equation}
    Next,
    \begin{align}
        \sum_{j=1}^D \norm{f_j}_{\V_\sigma(Q^d)}
        &= \sum_{j=1}^D \inf_{\substack{\nu_j \in \M(\Sph^d) \nonumber\\ f_j = f_{\nu_j}}} \norm{\nu_j}_{\M(\Sph^d)} \nonumber\\
        &\leq \inf_{\substack{\vec{\nu} \in \M(\Sph^d; \R^D) \nonumber\\ f = f_\vec{\nu}}} \sum_{j=1}^D \norm{\nu_j}_{\M(\Sph^d)} \nonumber\\
        &=  \inf_{\substack{\vec{\nu} \in \M(\Sph^d; \R^D) \nonumber\\ f = f_\vec{\nu}}} \norm{\vec{\nu}}_{\M, 1} \nonumber\\
        &\leq D^\frac{3}{2} \inf_{\substack{\vec{\nu} \in \M(\Sph^d; \R^D) \\ f = f_\vec{\nu}}} \norm{\vec{\nu}}_{2, \M} = D^\frac{3}{2} \norm{f}_{\V_\sigma(Q^d; \R^D)},
    \end{align}
    where the equalities of the form $f_j = f_{\nu_j}$ or $f = f_\vec{\nu}$ are understood as a function of $\vec{x} \in Q^d$. The fourth line follows from \cref{eq:equi-norms} combined with the inequality between $\|\dummy\|_1$ and $\|\dummy\|_2$.

    To prove the claim, given any $f \in \V_\sigma(Q^d; \R^D)$, we construct a $K$-term approximant $f_j^K$ for each component $f_j$, $j = 1, \ldots, D$, as in \cref{eq:MJB}. We then construct the vector-valued function $f_{DK} = (f_1^K, \ldots, f_D^K)$ which has, at most, $DK$ terms. This approximant satisfies
    \begin{align}
        \norm{f - f_{DK}}_{L^2(Q^d; \R^D)}^2
        &= \sum_{j=1}^D \norm{f_j - f_j^K}_{L^2(Q^d)}^2 \nonumber \\
        &\leq \sum_{j=1}^D C_0^2 C_\sigma^2 \norm{f_j}_{\V_\sigma(Q^d)}^2 K^{-1} \nonumber \\
        &\leq C_0^2 C_\sigma^2 K^{-1} \sum_{j=1}^D \norm{f_j}_{\V_\sigma(Q^d)}^2 \nonumber \\
        &\leq C_0^2 C_\sigma^2 K^{-1} \paren*{\sum_{j=1}^D \norm{f_j}_{\V_\sigma(Q^d)}}^2 \nonumber \\
        &\leq C_0^2 C_\sigma^2 D^3\norm{f}_{\V_\sigma(Q^d; \R^D)}^2 K^{-1}.
    \end{align}
    Therefore, 
    \begin{equation}
        \norm{f - f_{DK}}_{L^2(Q^d; \R^D)} \leq C_0 C_\sigma D^\frac{3}{2} \norm{f}_{\V_\sigma(Q^d; \R^D)} K^{-1/2},
    \end{equation}
    which proves the theorem.
\end{proof}

\section{Proof of \cref{thm:VV-rep}} \label{app:VV-rep}

\begin{proof}
    From the definition of the $\V_\sigma(\R^d; \R^D)$-norm \cref{eq:VV-norm}, it follows from a standard argument \citep[e.g.,][Proposition~3.7]{BartolucciRKBS} that the problem in \cref{eq:VV-opt} is equivalent to the problem
    \begin{equation}
        \inf_{\vec{\nu} \in \M(\Sph^d; \R^D)} \: \sum_{i=1}^N \mathcal{L}\paren*{\vec{y}_i, \int_{\Sph^d} \sigma(\vec{w}^\T\aug{\vec{x}}_i) \dd\vec{\nu}(\vec{w})} + \lambda \norm{\vec{\nu}}_{2, \M}
        \label{eq:opt-M2}
    \end{equation}
    in the sense that their infimal values are the same and if $\vec{\nu}^\star$ is a solution to \cref{eq:opt-M2}, then
    \begin{equation}
        f_{\vec{\nu}^\star}(\vec{x}) = \int_{\Sph^d} \sigma(\vec{w}^\T\aug{\vec{x}}_i) \dd\vec{\nu}^\star(\vec{w})
    \end{equation}
    is a solution to \cref{eq:VV-opt}. We now proceed in four steps to prove the theorem.

    \paragraph{Step (i): Existence of solutions to \cref{eq:opt-M2}.}
    
    Define
    \begin{equation}
        \mathcal{J}(\vec{\nu}) \coloneqq \sum_{i=1}^N \mathcal{L}\paren*{\vec{y}_i, \int_{\Sph^d} \sigma(\vec{w}^\T\aug{\vec{x}}_i) \dd\vec{\nu}(\vec{w})} + \lambda \norm{\vec{\nu}}_{2, \M}.
    \end{equation}
    Given an arbitrary $\vec{\nu}_0 \in \M(\Sph^d; \R^D)$, let $C_0 \coloneqq \mathcal{J}(\vec{\nu}_0)$. Then, we can transform \cref{eq:opt-M2} into the constrained problem
    \begin{equation}
        \inf_{\vec{\nu} \in \M(\Sph^d; \R^D)} \: \mathcal{J}(\vec{\nu}) \quad\textrm{s.t.}\quad \norm{\vec{\nu}}_{2, \M} \leq C_0 / \lambda.
        \label{eq:transformed}
    \end{equation}
    This transformation is valid since any measure that does not satisfy the constraint will have an objective value strictly larger than $\vec{\nu}_0$ and therefore will not be in the solution set. Next, we note that we can write
    \begin{equation}
        \mathcal{J}(\vec{\nu}) = \sum_{i=1}^N \mathcal{L}\paren*{\vec{y}_i, \HOp_i\curly{\vec{\nu}}} + \lambda \norm{\vec{\nu}}_{2, \M}
    \end{equation}
    where
    \begin{equation}
        \HOp_i\curly{\vec{\nu}} \coloneqq \begin{bmatrix}
            \ang{\vec{\nu}, h_{i, j}} \\
            \vdots \\
            \ang{\vec{\nu}, h_{i, D}}
        \end{bmatrix} \in \R^D
        \label{eq:pairing}
    \end{equation}
        with $h_{i, j}(\vec{w}) \coloneqq \sigma(\vec{w}^\T\aug{\vec{x}}_i) \vec{e}_j$, where $\vec{e}_j \in \R^D$ is the $j$th canonical unit vector.
        
        Note that $\sigma: \R \to \R$ is continuous by assumption. Therefore, we readily observe that $h_{i, j} \in C(\Sph^d; \R^D)$, the space of continuous functions on $\Sph^d$ taking values in $\R^D$. In \cref{eq:pairing}, $\ang{\dummy, \dummy}$ denotes the duality pairing\footnote{The continuous dual of $C(\Sph^d; \R^D)$ can be identified with $\M(\Sph^d; \R^D)$ by Singer's representation theorem~\citep{singer1957linear,singer1959applications} \citep[see also][]{hensgen1996simple,bredies2020higher}.} between $C(\Sph^d; \R^D)$ and $\M(\Sph^d; \R^D)$. Thus, for $i = 1, \ldots, N$, $\HOp_i$ is component-wise weak$^*$ continuous on $\M(\Sph^d; \R^D)$. Since $\mathcal{L}(\dummy, \dummy)$ is lower semicontinuous in its second argument combined with the fact that every norm is weak$^*$ continuous on its corresponding Banach space, we have that $\mathcal{J}$ is weak$^*$ continuous on $\M(\Sph^d; \R^D)$. By the Banach--Alaoglu theorem~\citep[Chapter~3]{rudin-functional}, the constraint set of \cref{eq:transformed} is weak$^*$ compact. Thus, \cref{eq:transformed} is the minimization of a weak$^*$ continuous functional over a weak$^*$ compact set. By the Weierstrass extreme value theorem on general topological spaces~\citep[Chapter~5]{convex-functional-analysis}, there exists a solution to \eqref{eq:transformed} (and subsequently of \eqref{eq:opt-M2}).

    \paragraph{Step (ii): Recasting \cref{eq:opt-M2} as an interpolation problem.}

    Let $\widetilde{\vec{\nu}}$ be a (not necessarily unique) solution to \cref{eq:opt-M2}, which is guaranteed to exist by the previous argumentation. For $i = 1, \ldots, D$, define
    \begin{equation}
        \vec{z}_i = \HOp_i\curly{\widetilde{\vec{\nu}}} \in \R^D.
    \end{equation}
    Then, $\widetilde{\vec{\nu}}$ must satisfy
    \begin{equation}
       \widetilde{\vec{\nu}} \in \argmin_{\vec{\nu} \in \M(\Sph^d; \R^D)} \norm{\vec{\nu}}_{2, \M} \quad\textrm{s.t.}\quad \HOp_i\curly{\vec{\nu}} = \vec{z}_i, \: i = 1, \ldots, N.
       \label{eq:interpolation-opt}
    \end{equation}
    To see this, we note that if this were not the case, it would contradict the optimality of $\widetilde{\vec{\nu}}$. This reduction implies that any solution to the interpolation problem \cref{eq:interpolation-opt} will also be a solution to \cref{eq:opt-M2}.

    \paragraph{Step (iii): The form of the solution.}

    We can rewrite the interpolation problem as
    \begin{equation}
       \min_{\vec{\nu} \in \M(\Sph^d; \R^D)} \norm{\vec{\nu}}_{2, \M} \quad\textrm{s.t.}\quad \ang{\vec{\nu}, h_{i, j}} = z_{i, j}, \: i = 1, \ldots, N \text{ and } j = 1, \ldots, D.
       \label{eq:interpolation-opt-2}
    \end{equation}
    This is the vector-valued analogue of the classical (Radon) measure recovery problem with $ND$ weak$^*$ continuous measurements \citep[cf.,][]{Zuh}. By the abstract representer theorems of~\citet{BoyerRepresenter,BrediesSparsity,UnserUnifyingRepresenter}, there always exists a solution to \cref{eq:interpolation-opt-2} that takes the form
    \begin{equation}
        \vec{\nu}^\star = \sum_{k=1}^K c_k \vec{e}_k
    \end{equation}
    with $K \leq ND$, $c_k \in \R \setminus \curly{0}$ where for $k = 1, \ldots, k$, $\vec{e}_k$ is an extreme point of the unit regularization ball
    \begin{equation}
        B \coloneqq \curly{\vec{\nu} \in \M(\Sph^d; \R^D) \st \norm{\vec{\nu}}_{2, \M} \leq 1}.
    \end{equation}
    From \citet[Theorem~2]{werner1984extreme}, the extreme points of $B$ take the form $\vec{a} \delta_\vec{w}$ with $\vec{a} \in \R^D$, $\norm{\vec{a}}_2 = 1$, and $\vec{w} \in \Sph^d$. Thus, we can write
    \begin{equation}
        \vec{\nu}^\star = \sum_{k=1}^K c_k \vec{a}_k \delta_{\vec{w}_k}
    \end{equation}
    with $\vec{a}_k \in \R^D$ and $\vec{w}_k \in \Sph^d$.
    By the equivalence between \cref{eq:opt-M2} and \cref{eq:VV-opt}, we find that there exists a solution to \cref{eq:VV-opt} that takes the form
    \begin{equation}
        \widehat{f}(\vec{x}) = \sum_{k=1}^{\widehat{K}} \widehat{\vec{v}}_k \sigma(\widehat{\vec{w}}_k^\T\aug{\vec{x}}), \quad \vec{x} \in \R^d,
    \end{equation}
    where $\widehat{K} \leq ND$, $\widehat{\vec{v}}_k \coloneqq c_k \vec{a}_k \in \R^D$ and $\widehat{\vec{w}}_k \in \Sph^d$.
    
    \paragraph{Step (iv): Sharpening the sparsity bound.}
    Step (iii) establishes that there exists an optimal solution 
    with $\widehat{K} \leq ND$ neurons. Thus, we can apply the same argumentation as in the proof of \cref{thm:widthbound} to find another solution where the number of neurons is bounded by the product of the rank of the labels $\mat{Y} \in \R^{D \times N}$ with the rank of the post-activation feature matrix $\mat{\Phi} \in \R^{\widehat{K} \times N}$. From the dimensions of these matrices, we see that the product of the ranks is $\leq N^2$. Therefore, there exists a solution to \cref{eq:VV-opt} that takes the form
    \begin{equation}
        f^\star(\vec{x}) = \sum_{k=1}^{K_0} \vec{v}_k \sigma(\vec{w}_k^\T\aug{\vec{x}}), \quad \vec{x} \in \R^d,
    \end{equation}
    where $K_0 \leq N^2$, $\vec{v}_k \in \R^D$ and $\vec{w}_k \in \Sph^d$. Thus, we can always find a solution to \cref{eq:VV-opt} with $\min\{N^2, ND\}$ neurons improving upon the bound of $ND + 1$ predicted by Carath\'eodory's theorem.
\end{proof}

\section{Proof of \cref{thm:deep-VV-rep}} \label{app:deep-VV-rep}
\begin{proof}
    Given $f = f^{(L)} \circ \cdots \circ f^{(1)}$ such that $f^{(\ell)} \in \V_{\sigma}(\R^{d_{\ell-1}}, \R^{d_\ell})$, $\ell=1,\cdots, L$, let the functional $\mathcal{J}$ denote the objective value of the optimization problem, i.e.,
    \begin{align}
        \mathcal{J}(f) \coloneqq \mathcal{J}(f^{(L)}, \ldots, f^{(1)}) \coloneqq \sum_{i=1}^{N} \mathcal{L} (\vy_i, f(\vx_i)) + \lambda \sum_{\ell=1}^{L} \|f^{(\ell)}\|_{\V_\sigma(\R^{d_{\ell-1}};\R^{d_{\ell}})}.
    \end{align}
    Next, following the same approach in the proof of \cref{thm:VV-rep}, for an arbitrary $g = g^{(L)} \circ \cdots \circ g^{(1)}$, where $g^{(\ell)} \in \V_\sigma(\R^{d_{\ell-1}}, \R^{d_{\ell}})$, $\ell=1,\cdots, L$, we define its objective value as $C \coloneqq \mathcal{J}(g)$. The unconstrained problem \cref{eq:deep-VV-opt} can then be transformed into the equivalent constrained problem
    \begin{equation}
        \inf_{\substack{f^{(\ell)} \in \V_\sigma(\R^{d_{\ell-1}}; \R^{d_\ell}) \\ \ell=1,\cdots,L \\f = f^{(L)} \circ \cdots \circ f^{(1)}}} \: \mathcal{J}(f) \quad \text{s.t.}  \quad \norm{f^{(\ell)}}_{\V_\sigma(\R^{d_{\ell-1}}; \R^{d_{\ell}})} \leq C/\lambda , \quad\ell=1,\cdots,L.
        \label{eq:const-deep-VV-opt}
    \end{equation}
    This transformation is valid since any collection of functions $f^{(\ell)}$, $\ell =1, \ldots, L$, that do not satisfy the constraints would result in an objective value that is strictly larger than that of $g$, and would therefore not be in the solution set.

    For any $f_0 = f_{0}^{(L)} \circ \cdots \circ f_{0}^{(1)}$, where $f^{(\ell)} \in \V_\sigma(\R^{d_{\ell-1}};\R^{d_{\ell}})$, $\ell=1,\cdots,L$, we will show that, for any fixed $\tilde{\ell} \in \curly{1, \ldots, L}$, the map $f^{(\tilde{\ell})}_{0} \mapsto \mathcal{J}(f_0)$ is weak$^*$ lower semicontinuous on $\V_\sigma(\R^{d_{\tilde{\ell}-1}}; \R^{d_{\tilde{\ell}}})$. First, observe that the map $f^{(\tilde{\ell})}_{0} \mapsto f_{0} (\vx_0)$, for any $\vx_0 \in \R^{d}$, is component-wise weak$^*$ continuous on $\V_\sigma(\R^{d_{\tilde{\ell}-1}};\R^{d_{\tilde{\ell}}})$. Indeed, this follows since for any $\bm{x}_0 \in \R^{d}$ the point evaluation operator
    \begin{align}
        \widetilde{\bm{x}}_0: f \mapsto f(\bm{x}_0) = \begin{bmatrix}
            f_1(\bm{x}_0)\\
            \vdots\\
            f_D(\bm{x}_0)
        \end{bmatrix}
    \end{align}
    is component-wise weak$^*$ continuous by Lemma~2.9 of \citet{parhi2022kinds} combined with the equivalence of norms in \cref{lemma:equivalent-VV-M-norms}. Thus, since $f_0^{(\widetilde{\ell})} \mapsto f^{(L)}_0 \circ \cdots \circ f^{(1)}_0(\bm{x}_0)$ is made up of compositions of component-wise continuous and component-wise weak$^*$ continuous functions. it is therefore itself component-wise weak$^*$ continuous on $\V_\sigma(\R^{d_{\widetilde{\ell}-1}}; \R^{d_{\widetilde{\ell}}})$. Therefore, the map $(f^{(1)}_0, \cdots, f^{(L)}_0) \mapsto \mathcal{J}(f_0)$ is weak$^*$ lower semicontinuous on $\V_\sigma(\R^{d_{\widetilde{\ell}-1}};\R^{d_{\widetilde{\ell}}})$. Finally, by the Banach--Alaoglu theorem \citep[Chapter 3]{rudin-functional}, the feasible set in \cref{eq:const-deep-VV-opt} is weak$^*$ compact. Therefore there exists a solution to \cref{eq:const-deep-VV-opt}, and thus \cref{eq:deep-VV-opt}, by the Weierstrass extreme value theorem on general topological spaces \citep[Chapter 5]{convex-functional-analysis}.

    To complete the proof, let $\widetilde{f} = \widetilde{f}^{(L)} \circ \cdots \circ \widetilde{f}^{(1)}$ be any solution to \cref{eq:deep-VV-opt}. By applying $\widetilde{f}$ to each data point $\bm{x}_i$, $i=1,\cdots,N$, we can recursively compute the intermediate vectors $\vz_{i,\ell} \in \R^{d_{\ell}}$ as follows:
    \begin{itemize}
        \item Initialize $\vz_{i,0} := \vx_i$.
        \item For each $\ell = 1,\cdots,L$, recursively update $\vz_{i,\ell} := \widetilde{f}^{\ell}(\vz_{i,\ell-1})$.
    \end{itemize}
    The solution $\widetilde{f}$ must satisfy
    \begin{align}\label{eq:layerwise-VV-opt}
        \widetilde{f}^{(\ell)} \in \arg \min_{f \in \V_\sigma(\R^{d_{\ell-1}}; \R^{d_{\ell}})} \|f\|_{\V_\sigma(\R^{d_{\ell-1}};\R^{d_{\ell}})} \quad \text{s.t.} \quad f(\vz_{i,\ell-1}) = \vz_{i,\ell}, i=1,\cdots,N,
    \end{align}
for $\ell=1,\cdots,L$ (since otherwise, it would contradict the optimality of $\tilde{f}$). By \cref{thm:VV-rep}, there always exists a solution to \cref{eq:layerwise-VV-opt} that takes the form of a shallow vector-valued neural network. Hence, there always exists a solution to \cref{eq:deep-VV-opt} of the form in \cref{eq:deep-VV-form}. 
\end{proof}

\section{Proof of \Cref{thm:multi_task_lasso_bound}}\label{app:multi_task_lasso}
 Our proof relies extensively on Carathéodory's theorem~\citep[Proposition 2.6]{clarke2013functional} which we first state here.
\begin{theorem}[Carathéodory's theorem]
    Let $S$ be a subset of a normed vector space with finite dimension $R$. Then  every  point $\vx \in \mathrm{Conv}(S)$, the convex hull of $S$, can be represented by a convex combination of at most $R+1$ points from $S$.
\end{theorem}
\begin{proof}[Proof of \Cref{thm:multi_task_lasso_bound}]
The constraint $\mathbf{\Psi} = \mathbf{V} \mathbf{\Phi}$ can be satisfied if and only if the row space of $\mathbf{\Psi}$ is contained in the row space of $\mathbf{\Phi}$.  This assumption, and the fact that the objective is continuous and coercive, implies a solution exists.
Suppose that $\mathbf{V}$ is a solution to our problem. Then we will show that there exists a (possibly different) solution $\widehat{\mathbf{V}}$ with no more than $r_{\mathbf{\Phi}} r_{\mathbf{\Psi}}$ nonzero columns.

Let $\text{col}(\mathbf{V})$ and $\text{col}(\mathbf{\Psi})$ denote the column space of $\mathbf{V}$ and $\mathbf{\Psi}$ respectively.
We first show that $\text{col}(\mathbf{V}) = \text{col}(\mathbf{\Psi})$.
Let $\mathbf{V} = \mathbf{A} + \mathbf{B}$, where $\mathbf{A} \in \text{col}(\mathbf{\Psi})$ and $\mathbf{B} \in \text{col}(\mathbf{\Psi})^\perp$, the subspace orthogonal to $\text{col}(\mathbf{\Psi})$.
Then,
\begin{align}
\sum^{K}_{k=1} \| \vv_k \|_2 = \sum^{K}_{k=1} \sqrt{ \| \va \|^2_2 + \|\vb\|^2_2 }.
\end{align}
Let $\mathbf{P}_{\text{col}(\mathbf{\Psi})}$ be the orthogonal projection onto $\text{col}(\mathbf{\Psi})$. We can then express the constraint as $\mathbf{\Psi} = \mathbf{P}_{\text{col}(\mathbf{\Psi})} \mathbf{\Psi}= \mathbf{P}_{\text{col}(\mathbf{\Psi})}(\mathbf{A}+\mathbf{B})\mathbf{\Phi}= \mathbf{A}\mathbf{\Phi}$.
Thus the solution must have $\mathbf{B}=0$, since anything nonzero would increase the objective without contributing to the constraint. Therefore, $\text{col}(\mathbf{V}) = \text{col}(\mathbf{\Psi})$ and $\text{rank}(\mathbf{V}) = \text{rank}(\mathbf{\Psi})$.
Now observe that we can also express the constraint as a sum of outer products,
\begin{align}
    \mathbf{\Psi} = \sum^{K}_{k=1} \vv_k \vec{\phi}_k^{T}
\end{align}
where $\vv_k$ are the columns of $\mathbf{V}$ and $\vec{\phi}^{T}_k$ are the rows of $\mathbf{\Phi}$. Let $\mathbf{M}_k =  \vv_k \vec{\phi}_k^{T}$. Since the $\vv_k$ belong to an $r_{\mathbf{\Psi}}$ dimensional subspace and the $\vec{\phi}_k$ belong to an $r_\mathbf{\Phi}$ dimensional subspace, the $\mathbf{M}_k$ all belong to a subspace of dimension at most $r_{\mathbf{\Phi}} r_{\mathbf{\Psi}}$. For ease of notation we let $R = r_{\mathbf{\Phi}} r_{\mathbf{\Psi}}$.
Now, define the optimal objective value $\gamma = \sum_{k=1}^{K} \| \vv_k \|_{2}$ and $\alpha_k = \|\vv_k\|_{2} / \gamma$ so that $\sum_{k=1}^{K} \alpha_k = 1$. We can then write $\mathbf{\Psi}$ as
\begin{align}
    \mathbf{\Psi} = \sum_{k=1}^{K} \alpha_k \left(\frac{1}{\alpha_k} \mathbf{M}_k \right) = \sum_{k=1}^{K} \alpha_k \widetilde{\mathbf{M}}_k.
\end{align}
This shows that $\mathbf{\Psi}$ is in the convex hull of matrices $\widetilde{\mathbf{M}}_k$ with dimension at most $R$.  Carathéodory's 
theorem implies that we can represent $\mathbf{\Psi}$ by a convex combination of a subset $\{ \widetilde{\mathbf{M}}_{j} \}_{j \in J}$ where $J \subset \{1,\cdots,K\}$ and $|J| \leq R+1$. Thus we can represent $\mathbf{\Psi}$ with no more than $R+1$ nonzero columns vectors in the solution
\begin{align}
    \mathbf{\Psi} &= \sum_{j \in J} \beta_j \widetilde{\mathbf{M}}_j \ = \ \sum_{j \in J} \beta_j \left(\frac{1}{\alpha_{j}} \vv_{j} \vec{\phi}^{T}_{j} \right).
\end{align}

Now to show that a solution exists with no more than $R$ nonzero columns we study the KKT conditions and apply  Carathéodory's theorem a second time.
Assume that $|J| = R+1$ and define $\widetilde{\vv}_{j} = \frac{1}{\alpha_{j}} \vv_{j}$. Thus  $\mathbf{\Psi}$ is in the convex hull of
\begin{align}
    \left\{\widetilde{\vv}_{j}\vec{\phi}^{T}_{j}\right\}_{j \in J}.
\end{align}
Each matrix $\widetilde{\vv}_{j}\vec{\phi}^{T}_{j}$ belongs to a subspace of dimension at most $R$, therefore, any matrix in this set can be expressed as a linear combination of the others
\begin{align}
\widetilde{\vv}_{i}\vec{\phi}^{T}_{i} = \sum_{j \in J ; j \neq i} c_j \widetilde{\vv}_{j}\vec{\phi}^{T}_{j}.
\end{align}
By the subgradient optimality conditions, we will prove that we must have $\sum_{j \in J ; j \neq i} c_j = 1$. This in turn implies that the set of matrices $\{\widetilde{\vv}_j\vec{\phi}^{T}_j\}_{j \in J}$ are not only linearly dependent but they also span an $R-1$ dimensional affine space (i.e., an $R-1$ dimensional hyperplane  not including the origin \citep{rockafellar1997convex}). We can then apply Carathéodory's again to show that a convex combination of $R$ matrices from this set suffices to satisfy the constraint.

The Lagrangian of our optimization problem \eqref{opt:multi-task-lasso} is
\begin{align}
    \mathcal{L}(\mathbf{V}, \vec{\nu}_1, \vec{\nu}_2) = \sum_{k=1}^{K} \|\vv_k\|_{2} + \vec{\nu}^{T}_1 (\mathbf{\Psi} - \mathbf{V}\mathbf{\Phi})\vec{\nu}_2
\end{align}
with Lagrange multipliers $\vec{\nu}_1 \in \R^{D}$ and $\vec{\nu}_2 \in \R^{N}$. By the KKT conditions, any solution must satisfy 
\begin{align}
    0 &\in \partial_{\mathbf{V}} \mathcal{L},\\
    0 &= \nabla_{\vec{\nu}_1, \vec{\nu}_2} \mathcal{L}.
\end{align}
Now for our original solution $\mathbf{V}$ we have,
\begin{align}
    \partial_{\mathbf{V}} \mathcal{L} &= \partial\left(\sum_{k=1}^{K} \|\vv_k\|_2\right) - \vec{\nu}_1(\mathbf{\Phi} \vec{\nu}_2)^{T}   = 0 \label{eq:first_kkt}\\
    \nabla_{\vec{\nu}_1,\vec{\nu}_2} \mathcal{L} &= \mathbf{\Psi} - \mathbf{V}\mathbf{\Phi} = 0.
\end{align}
Note we have no restriction of $\vec{\nu}_1$ and $\vec{\nu}_2$. Then for any $j \in J$ (the nonzero columns of $\mathbf{V}$) the $j$th column of the subgradient of our objective is
\begin{align}
    \partial_{\mathbf{V}} \left(\sum_{k=1}^{K} \|\vv_k\|_2 \right)_j =
        \frac{\vv_j}{\|\vv_j\|_2}.
\end{align}
Therefore, by \cref{eq:first_kkt} for the $j$th column of the subgradient we have 
\begin{align}
        \partial \left( \sum_{k=1}^{K} \| \vv_k \|_{2}\right)_j = \vec{\nu}_1 \vec{\nu}^{T}_2 \vec{\phi}_j \implies
        \frac{\vv_j}{\|\vv_j\|_2} = \vec{\nu}_1 \vec{\nu}^{T}_2 \vec{\phi}_j \label{eq:column_subgrad}
\end{align}
where $\vec{\phi}_j \in \R^{N \times 1}$ is the $j$th row of $\mathbf{\Phi}$.
Now consider the $i$th column of \cref{eq:first_kkt} and right-multiply both sides by $\widetilde{\vv}^{T}_i$ to obtain
 \begin{align}
 \frac{\vv_i \widetilde{\vv}^{T}_i}{\|\vv_i\|_2} = \vec{\nu}_1 \vec{\nu}^{T}_2 \vec{\phi}_i \widetilde{\vv}^{T}_i.
 \end{align}
By the linear dependence of the matrices $\{\widetilde{\vv}_j\vec{\phi}^{T}_j\}_{j \in J}$ for any $i \in J$ we have $\vec{\phi}_i \widetilde{\vv}^{T}_i =  \sum_{j \in J; j \neq i} c_j \vec{\phi}_j \widetilde{\vv}^{T}_j$ for some constants $c_j$, substituting this in we get
\begin{align}
        \frac{\vv_i \widetilde{\vv}^{T}_i}{\|\vv_i\|_2} &= \vec{\nu}_1 \vec{\nu}^{T}_2 \sum_{j \in J; j \neq i} c_j \vec{\phi}_j \widetilde{\vv}^{T}_j \nonumber\\
         &= \sum_{j \in J: j \neq i} c_j \vec{\nu}_1 \vec{\nu}^{T}_2   \vec{\phi}_j \widetilde{\vv}^{T}_j \nonumber\\
         &= \sum_{j \in J; j \neq i} c_j \frac{\vv_j\widetilde{\vv}^{T}_j}{\|\vv_j\|_2},
\end{align}
where the final equality follows from \cref{eq:column_subgrad}. Since $\widetilde{\vv}_j = \frac{1}{\alpha_j} \vv_j = \frac{\gamma}{\|\vv_j\|_2}\vv_j$ plugging this in on both sides and cancelling out $\gamma$ we have
\begin{align}
        \frac{\vv_i \vv^{T}_i}{\|\vv_i\|^2_2} &= \sum_{j \in J; j \neq i} c_j \frac{\vv_j\vv^{T}_j}{\|\vv_j\|^2_2}.
\end{align}
Taking the trace of the left and right hand sides reveals
\begin{align}
    \sum_{j \neq i} c_j = 1.
\end{align}
Therefore any $\vec{\phi}_i \widetilde{\vv}^{T}_i$ is in the affine hull of the other matrices and thus the matrices $\{\widetilde{\vv}_j\vec{\phi}^{T}_j\}_{j \in J}$ form an affine set of dimension $R-1$~\citep{rockafellar1997convex}. This implies that the matrices do not just span an $R$ subspace but rather an $R-1$ dimensional hyperplane that does not pass through the origin. Thus, the vector space corresponding to this set of matrices has dimension $R-1$. By applying
 Carathéodory's theorem again we can represent the constraint as a convex combination of just $R$ nonzero vectors
\begin{align}
    \mathbf{\Psi} &= \sum_{j \in J'} \hat{\beta}_j \widetilde{\vv}_j \vec{\phi}^{T}_j =\sum_{j \in J'} \vv^{*}_j \vec{\phi}^{T}_j, 
\end{align}
where $\vv^{*}_j:= \hat{\beta}_j \widetilde{\vv}_j$, $J' \subset \{1,\cdots, K\}$ and $|J'| \leq R = r_{\mathbf{\Phi}}r_{\mathbf{\Psi}}$. 
The coefficients $\hat{\beta}_j$ are nonnegative and sum to 1.
Finally, it remains to show that this new representation of $\mathbf{\Psi}$ has the same objective value as the original solution, indeed we have
\begin{align}
    \sum_{j \in J'} \|\vv^{*}_j\|_{2} = \sum_{j \in J'} \frac{\hat{\beta}_j}{\alpha_j} \|\vv_j\|_{2}  = \sum_{j \in J'} \gamma \hat{\beta}_j = \gamma
\end{align}
where we recall that $\gamma$ is the objective value obtained by our original assumed solution. 

The lower bound follows solely from the constraint. First we note that for any $\mathbf{\Psi}$ we always need at least $r_{\mathbf{\Psi}}$ nonzero columns in the solution $\mathbf{V}$. If this were not the case we would have $\text{rank}(\mathbf{V}) < r_{\mathbf{\Psi}}$ and therefore $\text{rank}(\mathbf{\Psi}) \leq \min\{\text{rank}(\mathbf{V}),r_{\mathbf{\Phi}}\} < r_{\mathbf{\Psi}} $ contradicting the fact that $\text{rank}(\mathbf{\Psi}) = r_{\mathbf{\Psi}}$.

Now for any $\mathbf{\Phi}$ the number of subspaces spanned by $p < r_{\mathbf{\Phi}}$ of the rows of $\mathbf{\Phi}$ is at most
\begin{align}
    \binom{K}{1} + \binom{K}{2} + \cdots + \binom{K}{r_{\mathbf{\Phi}} - 1}.
\end{align}
The union of these subspaces has Lebesgue measure zero in the $\text{row}(\mathbf{\Phi})$. Therefore, there exist uncountably many points in $\text{row}(\mathbf{\Phi})$ that cannot be represented with fewer than $r_\mathbf{\Phi}$ columns of $\mathbf{\Phi}$.


\end{proof}

\bibstyle{plain}
\bibliography{refs}

\end{document}